\documentclass{article}

\usepackage{microtype}
\usepackage{graphicx}
\usepackage{booktabs} 

\usepackage{hyperref}

\usepackage[accepted]{icml2023}

\usepackage{amsmath}
\usepackage{amssymb}
\usepackage{mathtools}
\usepackage{amsthm}

\usepackage[capitalize,noabbrev]{cleveref}

\usepackage[utf8]{inputenc} 
\usepackage[T1]{fontenc}    
\usepackage{hyperref}       
\usepackage{url}            
\usepackage{booktabs}       
\usepackage{amsfonts}       
\usepackage{nicefrac}       
\usepackage{microtype}      
\usepackage{xcolor}         

\usepackage{amsfonts}
\usepackage{bm}
\usepackage{graphicx}
\setcitestyle{round}
\usepackage{dsfont}
\usepackage{caption}
\usepackage{subcaption}
\usepackage{centernot}
\usepackage{wrapfig}
\usepackage{adjustbox}
\hypersetup{
    colorlinks=true,
    linkcolor=blue,
    citecolor=blue,      
    urlcolor=brown,
}
\usepackage{multirow}
\usepackage{enumitem}

\newcommand{\calX}{\mathcal{X}}
\newcommand{\calA}{\mathcal{A}}
\newcommand{\calB}{\mathcal{B}}
\newcommand{\calF}{\mathcal{F}}
\newcommand{\calH}{\mathcal{H}}
\newcommand{\calI}{\mathcal{I}}
\newcommand{\calK}{\mathcal{K}}
\newcommand{\calN}{\mathcal{N}}
\newcommand{\calS}{\mathcal{S}}
\newcommand{\calU}{\mathcal{U}}
\newcommand{\bbR}{\mathbb{R}}

\newcommand{\bbE}{\mathbb{E}}
\newcommand{\bbZ}{\mathbb{Z}}

\newcommand{\sd}{\mathbb{S}}
\newcommand{\ksd}{\mathbb{D}}
\newcommand{\ksdHat}{\hat{\mathbb{D}}}

\newenvironment{talign*}
 {\csname align*\endcsname}
 {\endalign}
\newenvironment{talign}
{\align}
{\endalign}

\newenvironment{proplist}{\begin{enumerate}[leftmargin=2.5em,labelwidth=0em,label=(\roman{enumi}),topsep=.1em,partopsep=0em,itemsep=-.3em]}{\end{enumerate}}

\theoremstyle{plain}
\newtheorem{theorem}{Theorem}[section]
\newtheorem{proposition}[theorem]{Proposition}
\newtheorem{lemma}[theorem]{Lemma}

\theoremstyle{definition}

\theoremstyle{remark}
\newtheorem{remark}[theorem]{Remark}

\usepackage[textsize=tiny]{todonotes}

\icmltitlerunning{Using Perturbation to Improve GOF Tests based on KSD}

\begin{document}

\twocolumn[
\icmltitle{Using Perturbation to Improve Goodness-of-Fit Tests based on Kernelized Stein Discrepancy}


\icmlsetsymbol{equal}{*}

\begin{icmlauthorlist}
\icmlauthor{Xing Liu}{ic}
\icmlauthor{Andrew B. Duncan}{ic,at}
\icmlauthor{Axel Gandy}{ic}
\end{icmlauthorlist}

\icmlaffiliation{ic}{Department of Mathematics, Imperial College London, London, UK.}
\icmlaffiliation{at}{Alan Turing Institute, London, UK}

\icmlcorrespondingauthor{Xing Liu}{xing.liu16@imperial.ac.uk}

\icmlkeywords{Goodness-of-fit testing, Stein's method, Kernel method}

\vskip 0.3in
]



\printAffiliationsAndNotice{} 

\begin{abstract}
  Kernelized Stein discrepancy (KSD) is a score-based discrepancy widely used in goodness-of-fit tests. It can be applied even when the target distribution has an unknown normalising factor, such as in Bayesian analysis. We show theoretically and empirically that the KSD test can suffer from low power when the target and the alternative distributions have the same well-separated modes but differ in mixing proportions. We propose to perturb the observed sample via Markov transition kernels, with respect to which the target distribution is invariant. This allows us to then employ the KSD test on the perturbed sample. We provide numerical evidence that with suitably chosen transition kernels the proposed approach can lead to substantially higher power than the KSD test.
\end{abstract}

\section{Introduction}
Stein discrepancy (SD) \citep{stein1972bound, gorham2015measuring} is a statistical divergence between two probability measures based on Stein's method. More specifically, given two Borel probability measures $Q$ and $P$ supported on $\calX \subset \mathbb{R}^d$, the Stein discrepancy is defined to be
\begin{talign}
\label{eq: stein discrepancy}
    \ksd_\mathcal{F}(Q, P) \coloneqq \sup_{f\in \calF}\mathbb{E}_{x\sim Q}[\calA_P f(x)],
\end{talign}
where $\calF$ is a set of functions on $\calX$ and $\mathcal{A}_P$ is an operator acting on $\calF$ such that $\mathbb{E}_{x\sim Q}[\calA_P f(x)] = 0$ for all $f \in \mathcal{F}$ if and only if $Q\equiv P$. When $\calX = \mathbb{R}^d$ and $P$ admits a positive, continuously differentiable density $p$ with respect to the Lebesgue measure, then the \emph{Langevin-Stein operator} is the natural candidate for $\mathcal{A}_P$ which has the crucial property that it only depends on the \emph{score function} $s_p(x) \coloneqq \nabla \log p(x)$ of $p$, which does not require evaluation of the (possibly intractable) normalising constant of $p$.   When a Reproducing Kernel Hilbert Space (RKHS) \citep{berlinet2011reproducing} is used to construct the function class, this discrepancy is called the \emph{kernelized Stein discrepancy} (KSD), and admits a closed-form expression.  This has made KSD popular for applications involving an unnormalised density, such as Bayesian inference \citep{liu2016stein}, goodness-of-fit testing \citep{liu2016kernelized, chwialkowski2016kernel, jitkrittum2017linear}, and sample quality measurement \citep{gorham2015measuring, gorham2017measuring, gorham2019measuring}; see \citet{anastasiou2021stein} for a review.

We focus on goodness-of-fit (GOF) testing, where independent samples from a \emph{candidate distribution} $Q$ are observed and the goal is to test for evidence against the null hypothesis that $Q$ matches a \emph{target distribution} $P$. When the density $p$ of $P$ is only available in an unnormalised form (i.e.\  the normalisation constant is infeasible to compute) and direct sampling from $P$ is infeasible, classical tests such as the Kolmogorov-Smirnov test \citep{massey1951kolmogorov} or two-sample tests \citep{gretton2012kernel,schrab2021mmd} cannot be used, as they require either a tractable cumulative distribution function or samples from $P$.  A GOF test based on KSD, on the other hand, does not have these limitations.

However, KSD tests may suffer from low test power when the target probability measure has well-separated modes. For example, when $Q$ and $P$ are mixtures of the same components but differ in the mixing proportions,  the power will converge to the test level as the modes of the components become more and more separated (Fig.~\ref{fig: bimodal delta and densities}). This is because the KSD statistic can be numerically close to 0 if the region where the score functions are practically different has low $Q$-probability. This issue, sometimes known as the \emph{blindness to isolated components} \citep{wenliang2020blindness}, has been noted in a number of works \citep{gorham2019measuring, matsubara2021robust, kanagawa2022controlling} and addressed in some applications of KSD \citep{zhang2022towards}. However, little work has been devoted to tackling this issue in the context of GOF testing. 

\vspace{-0.2cm}
\paragraph{Contributions}
Our contribution is twofold. First, we demonstrate theoretically and numerically with a bimodal Gaussian example that the power of KSD tests can converge to the test level when the target distribution has well-separated modes. This is different from the works of \citet{gorham2017measuring} and \citet{wenliang2020blindness}, which focus on the convergence of the sample KSD but not its limiting null distribution. Second, we address this issue by introducing a \emph{perturbation operator}, giving rise to a family of perturbation-based GOF test (Fig.~\ref{fig: bimodal delta and densities}, bottom right) which we call the \emph{perturbed kernelized Stein discrepancy} (pKSD) test. The role of the operator is to perturb the candidate and the target distributions simultaneously to create discrepancy that can be more easily detected by KSD. We propose to use Markov transition kernels that are invariant to the target $P$ as the perturbation operator. The $P$-invariance ensures the resulting GOF tests provably control the Type-I error. The transition kernel is non-irreducible and uses an inter-modal jump proposal, which can increase the test power against multi-modal alternatives, sometimes substantially from the nominal level to almost 1.

\vspace{-0.25cm}
\paragraph{Outline}
Section~\ref{sec: background} reviews kernelized Stein discrepancy. Section~\ref{sec: Limitations of KSD as a Score-Based Discrepancy} formalises the low-power problem of the KSD test. The proposed method is presented in Section~\ref{sec: proposed method} and Section~\ref{sec:transition kernel}. We discuss related work in Section~\ref{sec: related works}, followed by experiments in Section~\ref{sec: experiments}. Section~\ref{sec: conclusion} concludes.

\vspace{-0.25cm}
\paragraph{Notation}
Throughout this article, we denote by $Q, P$ probability measures on $\calX = \bbR^d$ equipped with the Borel $\sigma$-algebra $\calB(\calX)$, and assume $P$ has a continuously differentiable, positive Lebesgue density $p$.  We refer to $Q$ as the \emph{candidate} distribution and $P$ as the \emph{target} distribution. Our interest lies in testing $H_0: Q = P$ against $H_1: Q \neq P$ using a finite sample $\{ x_i \}_{i = 1}^n$ drawn independently from $Q$. We assume we can evaluate pointwise the \emph{unnormalised} density $p^\ast(x) = p(x)/Z$, where $Z$ is an unknown constant, as well as $\nabla \log p^\ast(x) $, which is identical to the \emph{score function} of $p$, namely $s_p(x) \coloneqq \nabla \log p(x) = ( \nabla_{x_1} \log p(x), \ldots, \nabla_{x_d} \log p(x) )^\top$.


\begin{figure}
    \centering
    \includegraphics[width=.4\textwidth]{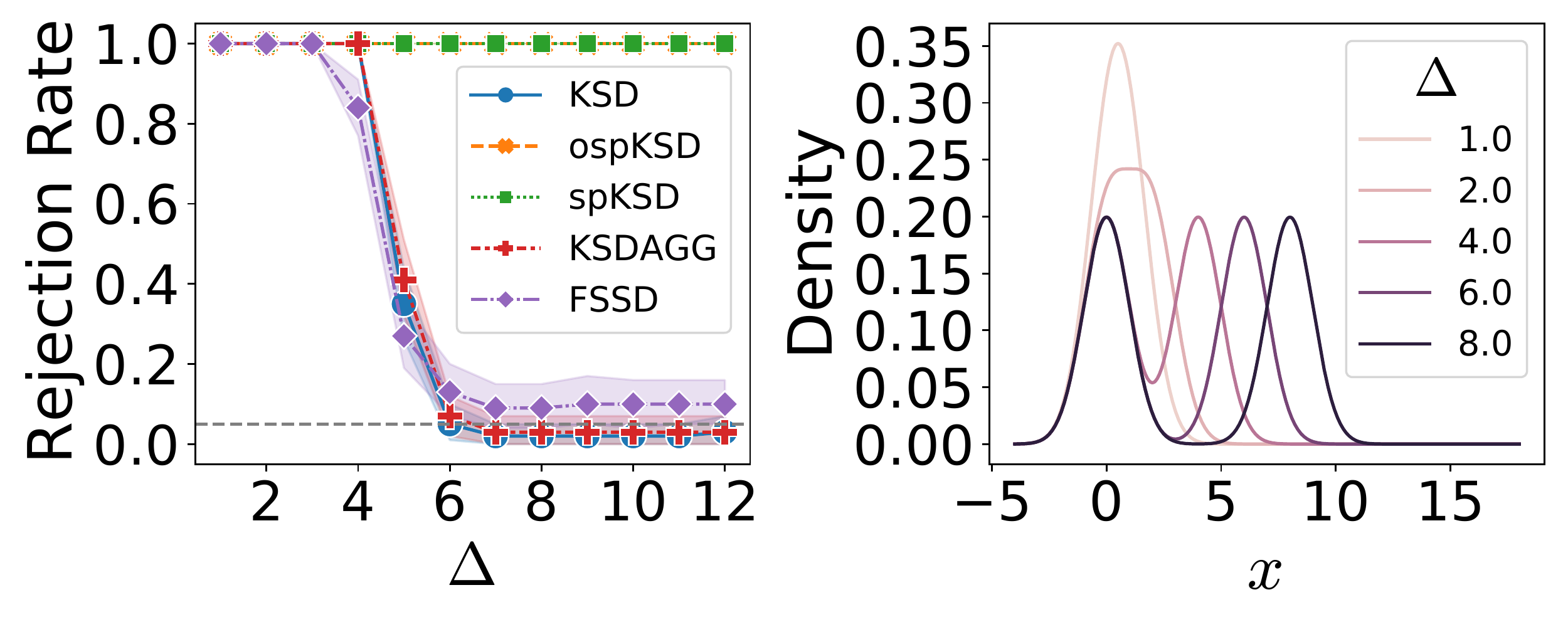}
    \includegraphics[width=.2\textwidth]{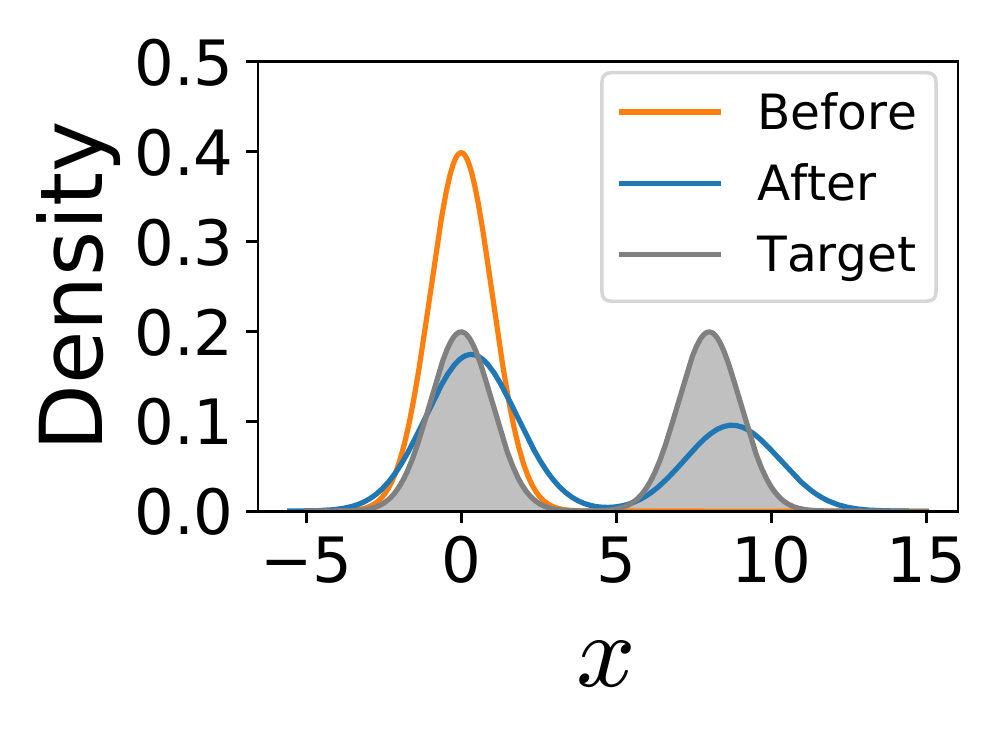}
    \includegraphics[width=.2\textwidth]{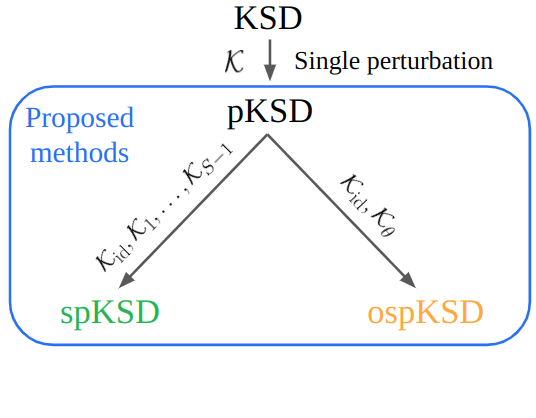}
    \vspace{-0.5cm}
    \caption{Power for a one-dimensional bimodal Gaussian target distribution $P$ with mixing weight $0.5$ and mode separation $\Delta$. The candidate distribution $Q$ is only the left component, from which 1000 samples are drawn. ospKSD and spKSD are our proposed method; the others are existing benchmarks. \emph{Top}: Rejection rates and target densities for varying $\Delta$; the orange and green lines overlap. \emph{Bottom left}: Density of $Q$ before and after $10$ steps of the perturbation described in Sec.~\ref{sec: proposed method} and density of the target $P$. \emph{Bottom right}: Connections between KSD and the proposed divergences.}
    \label{fig: bimodal delta and densities}
    \vspace{-1em}
\end{figure}

\section{Kernelized Stein Discrepancy Test}
\label{sec: background}
Choosing the Stein operator $\calA_P$ in \eqref{eq: stein discrepancy} to be the operator mapping continuously differentiable, vector-valued functions $f: \bbR^d \to \bbR^d$ to scalar-valued functions via $\calA_p f(x) \coloneqq \langle \nabla \log p(x), f(x) \rangle + \langle \nabla, f(x) \rangle $, one obtains a statistical divergence which only depends on the score function of $P$. If $f \in \calF$ satisfies regularity conditions such as $\lim_{\|x\|_2 \to \infty} f(x) p(x) = 0$, then one can show that $\bbE_{x \sim P} [\calA_P f(x)] = 0$, and $f$ is said to lie in the \emph{Stein class} of $P$ \citep[Sec.~2.2]{liu2016kernelized}. The function class $\calF$ is usually chosen to be \emph{(i)} sufficiently broad so that the discrepancy separates distinct probability measures, i.e.\ $\sd(Q, P; \calF) = 0 \iff Q = P$, and \emph{(ii)} sufficiently regular so that the right-hand-side of (\ref{eq: stein discrepancy}) can be efficiently solved. 

To this end, \citet{liu2016kernelized, chwialkowski2016kernel} proposed to let $\calF$ be the unit ball of a \emph{reproducing kernel Hilbert space} (RKHS) \citep{berlinet2011reproducing}. Specifically, let $\calH$ be an RKHS associated with positive definite kernel $k: \calX \times \calX \to \bbR$. Let $\calF^d$ be the unit ball of the $d$-times Cartesian product $\calH^d \coloneqq \calH \times \cdots \times \calH$. Choosing $\calF = \calF^d$ and the operator $\calA_P$ yields the \emph{(Langevin) kernelized Stein discrepancy} (KSD): $\ksd(Q, P) \coloneqq \ksd_{\calF^d}(Q, P)$.

Assuming the kernel $k$ has continuous first-order derivatives with respect to both arguments, \citet[Thm.\ 2.1]{chwialkowski2016kernel} showed that KSD attains a closed form: $\ksd(Q, P) = \mathbb{E}_{x, x' \sim Q}[ u_P(x, x') ]$, 
where $x, x'$ are independent random variables drawn from $Q$, and $u_P$ is the \emph{Stein kernel}: $u_P(x, x') \coloneqq s_p(x)^\top k(x, x') s_p(x') + s_p(x)^\top \nabla_{x'} k(x, x') + \nabla_x k(x, x')^\top s_p(x') + \sum_{i=1}^d \frac{\partial^2}{\partial x_i \partial x'_i}  k(x, x')$. Notably, $u_P$ (hence also $\ksd(Q, P)$) depends on $p$ only through $s_p(x) = \nabla \log p(x)$, so KSD is computable even without the knowledge of the normalising constant of $p$. 

We will assume $k$ lies in the \emph{Stein class} of $p$ \citep[Def.~3.4]{liu2016kernelized}, so that $\ksd(P, P) = 0$. When $Q$ also admits a density $q$,
and $k$ is $cc$-universal \citep{Sriperumbudur2011universality, Sriperumbudur2010relation} or integrally strictly positive definite \citep[Sec.~6]{stewart1976positive}, KSD is \emph{separating}, meaning that $\ksd(Q, P) = 0 \iff Q = P$, provided that $\bbE_{x \sim Q}[ \| s_q(x) - s_p(x) \|^2 ] < \infty$ \citep{chwialkowski2016kernel, liu2016kernelized}. The assumption that $Q$ has a density can be relaxed if the target density satisfies additional tail conditions, such as distant dissipativeness \citep[Proposition 4]{hodgkinson2020reproducing}. 

$\ksd(Q, P)$ can be estimated from a sample $\{ x_i \}_{i = 1}^n$ from $Q$ by the following U-statistic \citep[Sec.~5.5]{serfling2009approximation}:
\begin{talign}
    \ksdHat_P
    \coloneqq \frac{1}{n(n - 1)} \sum_{1 \leq i \neq j \leq n} u_P(x_i, x_j) .
    \label{eq: ksd u stat}
\end{talign}
The KSD test uses (\ref{eq: ksd u stat}) as a test statistic. The asymptotic distribution of $\ksdHat_P$ under $H_0$ has no closed form, but can be approximated with a bootstrap procedure \citep{huskova1993consistency} using the bootstrap samples
\begin{talign}
    \ksdHat^b_P \coloneqq \frac{1}{n^2} \sum_{1 \leq i \neq j \leq n} \left( w_i^b - 1\right) \left(w_j^b - 1 \right) u_P(x_i, x_j) ,
    \label{eq: bootstrap sample}
\end{talign}
where $(w_1^b, \ldots, w_n^b) \sim \textrm{Mult}\left(n; \frac{1}{n}, \ldots, \frac{1}{n}\right)$ follows a multinomial distribution. The test statistic $\ksdHat_P$ is compared against quantiles of $\{ \ksdHat^b_P \}_{b=1}^B $ computed with $B$ i.i.d.\ draws $(w_1^b, \ldots, w_n^b)$, and $H_0$ is rejected for large values of $\ksdHat_P$. The resulting test achieves the desired level $\alpha$ asymptotically \citep{huskova1993consistency, liu2016kernelized}.

Many improvements over the standard KSD test have been proposed, e.g., to reduce the computational cost \citep{jitkrittum2017linear}, to address the curse-of-dimensionality \citep{gong2021sliced, gong2021active}, and to avoid kernel selection by adopting an aggregated testing procedure \citep{schrab2022ksd}.

\section{Limitations of KSD Test}
\label{sec: Limitations of KSD as a Score-Based Discrepancy}
The KSD can be blind to certain discrepancies that are strongly visible in other metrics (e.g., in the $L_2$ norm). One example is mixtures of the same well-separated components, differing only by the mixing proportions (weights). 
In fact, the KSD will be small in settings where the score difference $\| s_p(x) - s_q(x) \|_2^2$ is large only with low $Q$-probability. This is because the KSD can be bounded from above by the \emph{Fisher Divergence} (FD) $F(q, p) \coloneqq \bbE_{x \sim Q}[ \| s_p(x) - s_q(x) \|_2^2 ]$ \citep[Thm.~5.1]{liu2016kernelized}.

This is known as the ``blindness'' of score-based discrepancies \citep{wenliang2020blindness} such as  KSD. This limitation of KSD has been highlighted in a number of works \citep{gorham2019measuring, matsubara2021robust, zhang2022towards, kanagawa2022controlling}; however, its implication to the test power in GOF tests has not yet been formalised.  

In Prop.~\ref{prop: distribution under H1 converges to null distribution} (proved in Appendix~\ref{proof: distribution under H1 converges to null distribution}), we formally connect the blindness issue with the rates of increase of the sample size and the FD between the two distributions. 



\begin{proposition}
    \label{prop: distribution under H1 converges to null distribution}
    Let $Q$ and $P_\nu$, $\nu = 1, 2, \ldots$, be probability measures defined on $\bbR^d$ with positive densities $q$ and $p_\nu$. respectively. Assume $\bbE_{x \sim Q}[ \| s_q(x) \|_2^2 ],\; \bbE_{x, x' \sim Q}[ u_Q(x, x')^2 ] < \infty$, and the kernel $k$ satisfies
    \begin{talign}
        \max\big\{ 
        & \bbE_{x, x' \sim Q}[ | k(x, x') | ], \; 
        \bbE_{x, x' \sim Q}[ \| \nabla_{x'} k(x, x')\|_2^2], \; \nonumber \\
        & \bbE_{x, x' \sim Q}[ \| \nabla_{x } k(x, x')\|_2^2]
        \big\} < \infty \;.
        \label{eq: boundedness assumption}
    \end{talign}
    Let $x_1,x_2, \ldots$ be a sequence of i.i.d.\ samples from $Q$. Denote by $F_\nu \coloneqq \bbE_{x \sim Q}[ \| s_{p_\nu}(x) - s_q(x) \|_2^2 ]$ the Fisher Divergence between $Q$ and $P_\nu$. If the sequence $n_1, n_2, \ldots \in \mathbb{N}$ satisfies $n_\nu \to \infty$ as $\nu \to \infty$ and $n_\nu = o( 1 / \max(F_\nu, F_\nu^{1/2}) )$, then
    \begin{talign}
        n_\nu \ksdHat_{P_{\nu}}
        \to_d \sum_{j = 1}^\infty c_j (Z_j^2 - 1) 
        \quad(\nu\to \infty) \;,
        \label{eq: limiting distribution ksdHat(Q, P_Delta)}
    \end{talign}
     where $\ksdHat_{P_{\nu}}$ is the sample KSD computed using $x_1, \ldots, x_{n_\nu}$, $Z_j \sim \calN(0, 1)$ i.i.d.\ and $\{ c_j \}$ are the eigenvalues of the Stein kernel $u_P$ under $Q$.
\end{proposition}

\begin{remark}
    The RHS of (\ref{eq: limiting distribution ksdHat(Q, P_Delta)}) is the limiting distribution of $\ksdHat_{P_\nu}$ under $H_0$ \citep{liu2016kernelized}. Hence, this result shows that if the sample size $n_\nu$ is $o(1 / \max(F_\nu, F_\nu^{1/2}))$, then the test power converges to the \emph{nominal level} of the test.
\end{remark}
\begin{remark}
    Assumption (\ref{eq: boundedness assumption}) is standard and holds for Inverse Multi-Quadrics (IMQ) and Radial Basis Function (RBF) kernels when $Q$ has a finite second moment. IMQ kernels are preferred as they have desired tail properties to ensure a \emph{convergence determining} KSD for target densities satisfying the distantly dissipative condition \citep{gorham2017measuring, hodgkinson2020reproducing}. This includes Gaussian mixtures with common covariance, as well as distributions strongly log-concave outside of a compact set, such as Bayesian linear, logistic, and Huber regression posteriors with Gaussian priors, c.f., \citet{gorham2019measuring,gorham2017measuring}. Prop.~\ref{prop: distribution under H1 converges to null distribution} does not contradict this result, as it considers a different regime where a \emph{sequence} of target distributions is of interest.
\end{remark}

Prop.~\ref{prop: distribution under H1 converges to null distribution} allows us to study the test power by analysing the FD. For instance, when $P$ is a mixture of two Gaussian components and $Q$ is one of its component, the FD decreases exponentially fast to 0 with the mode separation. Prop.~\ref{prop: distribution under H1 converges to null distribution} then implies that an unrealistically large sample size would be needed for the test to have a non-trivial power. This is formalised in the following result.

\begin{theorem}
\label{thm: multimodal}
    Let $Q = \calN( 0, I_d)$ and $P_\nu = \pi \calN( 0, I_d) + (1 - \pi) \calN( \Delta_\nu, I_d)$, where $\pi \in [0, 1]$ and $\Delta_\nu \in \bbR^d$. With the same notation in Prop.~\ref{prop: distribution under H1 converges to null distribution} and assuming $k$ satisfies (\ref{eq: boundedness assumption}), the limit (\ref{eq: limiting distribution ksdHat(Q, P_Delta)}) holds if $n_\nu = o\left(e^{\|\Delta_\nu \|_2^2 / 64} \right)$.
\end{theorem}
The proof is in Appendix~\ref{proof: multimodal}. Figure~\ref{fig: bimodal delta and densities} provides numerical evidence for Thm.~\ref{thm: multimodal} by showing the rate of rejection over $100$ repetitions at level $\alpha = 0.05$. We observe that the power of the KSD test (with IMQ kernel whose bandwidth is chosen by median heuristic \citep{gretton2012kernel}) approaches the prescribed level for $\Delta \geq 6$. A similarly poor performance is observed for \textsc{KSDAgg} \citep{schrab2022ksd} and FSSD \citep{jitkrittum2017linear}, two variants of KSD. In comparison, our proposed test, called ospKSD and spKSD, achieve an almost perfect power. Notably, the problem of low test power persists even if the samples are drawn from both components but with a different weight; see Figure~\ref{fig: bimodal ratio_s} in Sec.~\ref{sec: experiments}.

\section{KSD Test with Perturbation}
\label{sec: proposed method}
We propose to increase the power of KSD test against \emph{multi-modal alternatives} by perturbing both the candidate and the target distributions with a set of \emph{Markov transition kernels} \citep[Chapter 6]{robert2004monte} and performing KSD tests on the \emph{perturbed} distributions. A Markov transition kernel is a function $\calK: \calX \times \mathcal{B}(\calX) \to [0, 1]$ such that \emph{(i)} for all $x \in \calX$, $\calK(x, \cdot)$ is a probability measure on $(\calX, \mathcal{B}(\calX))$, and \emph{(ii)} for all $A \in \mathcal{B}(\calX)$, $\calK(\cdot, A)$ is a measurable function on $\calX$. In our example, $\calK$ may also be an iterated composition of an underlying kernel, e.g.\ a Metropolis-Hastings kernel. The perturbed measure of $Q$ is $(\calK Q)(\cdot) \coloneqq \int_{\calX} \calK (x, \cdot) Q(dx)$, and similarly for $\calK P$.

\subsection{KSD with a Single Perturbation Kernel}
\label{subsec: perturbed kernel stein discrepancy}
We first consider a \emph{single} transition kernel $\calK$. We define the \emph{perturbed kernelized Stein discrepancy} (pKSD) as
\begin{talign}
    \vspace{-0.2cm}
    \ksd(Q, P; \calK)
    &\coloneqq \ksd( \calK Q, \calK P) \nonumber \\
    &= \sup_{f \in \calF^d} | \bbE_{x \sim \calK Q}[\calA_{\calK P} f(x)] | \;,
    \label{eq: pKSD}
\end{talign}
assuming $\calK P$ admits a continuously differentiable density so that its score function is well-defined. Notably, $\calK Q$ need not have a (Lebesgue) density for (\ref{eq: pKSD}) to exist.

The properties of pKSD are dictated by the operator $\calK$. A desirable choice should ensure that \emph{(i)} pKSD is well-defined, and in particular $\calK P$ should have a continuously differentiable density whenever $P$ does, \emph{(ii)} pKSD (\ref{eq: pKSD}) can be computed efficiently, and \emph{(iii)} the test can achieve a high power against alternatives with wrong mixing weights.

Given these \emph{desiderata}, we propose to choose a transition kernel $\calK$ that is \emph{$P$-invariant}, i.e.\ $P(\cdot) = \int_\calX \calK(x, \cdot) p(x) dx$. A $P$-invariant kernel ensures $\calK P = P$, so the score function $s_{\calK p} = s_p$ is unchanged after perturbation. This means \emph{(i)} and \emph{(ii)} are trivially satisfied. In particular, pKSD will have a closed-form expression
\begin{align*}
    \vspace{-0.2cm}
    \ksd(Q, P; \calK) = \mathbb{E}_{x, x' \sim \calK Q }[u_{P}(x, x')] \;,
\end{align*}
provided that $\bbE_{x \sim \calK Q}[ u_{P}(x, x) ] < \infty$ (e.g., \citet[Thm.~2.1]{chwialkowski2016kernel}). Moreover, the $P$-invariance allows a GOF test similar to the standard KSD test to be constructed, as we will elucidate in Sec.~\ref{sec: pksd with multiple transition kernels}. To address \emph{(iii)}, we employ a proposal map for $\calK$ that ``aggregates'' densities across the modes of the distribution. As we will demonstrate numerically, such a proposal is sensitive to discrepancies in mixing weights.


Given i.i.d.\ $\{ x_i \}_i^n \sim Q$, a sample $\{ \tilde{x}_i \}_{i=1}^n$ from $\calK Q$ can be drawn by running 1-step transitions under $\calK$ starting from each $x_i$. pKSD can then be estimated by the U-statistic:
\begin{talign}
    \ksdHat_{P, \calK}
    \coloneqq \frac{1}{n(n - 1)} \sum_{1 \leq i \neq j \leq n} u_P(\tilde{x}_i, \tilde{x}_j) \;.
    \label{eq: pksd u stat}
    \vspace{-0.2cm}
\end{talign}

\begin{algorithm}[!t]
\caption{Goodness-of-Fit Test with spKSD.}
\label{alg: GoF with pKSD}
    \begin{algorithmic}
        \STATE {\bfseries Input:} Target $P$, observed sample $\{ x_i \}_{i=1}^n$ from $Q$, set of transition kernels $\calS = \{ \calK_s \}_{s=1}^S$ that includes $\calK_\textrm{id}$ (i.e., no perturbation), number of transition steps $T$.
        \STATE Estimate the mode $\{ \mu_1, \ldots, \mu_M \}$ and Hessians $\{ A_1, \ldots, A_M \}$ using Algorithm~\ref{alg: find mode approximations} in the Appendix.
        \STATE For $s = 1, \ldots, S$, perturb $\{ x_i \}_{i=1}^n$ with $\calK_s$ by $T$ steps to generate perturbed samples $\{ x_i^s \}_{i=1}^n$.
        \STATE Compute test statistic $\hat{\ksd}_{P, \calS}$ using (\ref{eq: spKSD u statistic}).
        \STATE Generate bootstrap samples with (\ref{eq: bootstrap sample}) with $u_P$ replaced by $\tilde{u}_P$, and find the $(1-\alpha)$-quantile $\hat{\gamma}_{1 - \alpha}$.
        \STATE Reject $H_0$ if $\ksdHat_{P, \calS} \geq \hat{\gamma}_{1 - \alpha}$.
    \end{algorithmic}
\end{algorithm}

\vspace{-0.14cm}
\subsection{KSD with Multiple Perturbation Kernels}
\label{sec: pksd with multiple transition kernels}
A single transition kernel can be limited in improving the test power against general multi-modal alternatives. It also does \emph{not} guarantee the \emph{separation} property, as $\ksd(\calK Q, \calK P) = 0 \centernot\implies Q = P$, unless $\calK$ is injective so that $\calK Q = \calK P \implies Q = P$ (such as the convolution operator). However, choosing only injective $\calK$ would significantly restrict the class of possible options. Instead, we propose to employ a finite \emph{collection} $\calS = \{ \calK_s \}_{s=1}^S$ of $P$-invariant transition kernels, and require $\calS$ to include the identity transition kernel $\calK_\textrm{id}$, defined as $\calK_\textrm{id}(x, A) = \delta_x(A)$ for all $x \in \calX$ and $A \in \calB(\calX)$, where $\delta_x(A) = 1$ if $x \in A$ and 0 otherwise. In particular, $\ksd(Q, P; \calK_{\textrm{id}})$ reduces to the standard KSD. This gives rise to a separating statistical divergence which we term \emph{sum-pKSD} (spKSD)
\begin{talign*}
    \ksd(Q, P; \calS)
    \coloneqq \sum_{\calK \in \calS} \ksd(\calK Q, P) \;,
\end{talign*}
where we have overloaded $\ksd(Q, P; \calS)$ with a set $\calS$ in place of a single transition kernel to denote spKSD. The next result (proved in Appendix~\ref{proof: spksd validity}) shows that spKSD indeed separates probability measures so long as $\calK_{\textrm{id}} \in \calS$.


\begin{proposition}[spKSD separation]
\label{prop: spKSD validity}
    Suppose $Q, P$ are probability measures on $\calX$ that admit positive (Lebesgue) densities $q, p$, respectively. Further assume $\bbE_{x \sim \calK Q}[u_P(x, x)] < \infty$ for all $\calK \in \calS$ and $\bbE_{x \sim Q}[\| s_p(x) - s_q(x) \|_2^2] < \infty$. If the kernel $k$ is cc-universal and $\calK_\textrm{id} \in \calS$, then $\ksd(Q, P; \calS) \geq 0$ with equality if and only if $Q = P$.
\end{proposition}
The assumption that the alternative distribution $Q$ also admits a density is common in KSD literature when proving separation (e.g., \citet{liu2016kernelized, chwialkowski2016kernel, jitkrittum2017linear, gong2021active}), but it can be relaxed if $P$ is light-tailed or distantly dissipitative \citep{hodgkinson2020reproducing, gorham2017measuring}.

spKSD can also be written as a double expectation akin to KSD, provided $\bbE_{x \sim \calK_s Q}[ u_{P}(x, x) ] < \infty$ for all $s$. This allows spKSD to be estimated given a random sample $\{ x_i \}_{i=1}^n$ from $Q$. Formally, for each $\calK_s \in \calS = \{ \calK_1, \ldots, \calK_S \}$, a sample $\{ x_i^s \}_{i=1}^n$ from $\calK_s Q$ can be drawn by running 1-step transitions under $\calK_s$ starting from each $x_i$. Denote by $x_i^{1:S} \coloneqq \textrm{concat}(x_i^1, \ldots, x_i^S)$ the concatenation of $x_i^1, \ldots, x_i^S$ into a single vector. We propose to estimate $\ksd(Q, P; \calS)$ using the following U-statistic
\begin{talign}
    \vspace{-0.2cm}
    \hat{\ksd}_{P, \calS}
    &\coloneqq 
    \frac{1}{n(n - 1)} \sum_{1 \leq i \neq j \leq n} \tilde{u}_P(x_i^{1:S}, x_j^{1:S}) \;,
    \label{eq: spKSD u statistic}
\end{talign}
where $\tilde{u}_P(x_i^{1:S}, x_j^{1:S}) \coloneqq \sum_{s=1}^S u_P(x_i^s, x_j^s)$.

\subsection{GOF Testing with spKSD}
\label{subsec: pKSD with Markov Transition Kernels}

Having constructed a test statistic for spKSD in the form of a U-statistic, the next result (proved in Appendix~\ref{proof: asymptotic distribution of spKSD}) derives the limiting distribution of spKSD statistic under the null and alternative hypotheses. We denote by $R_Q$ the distribution of $x_i^{1:S}$ constructed as before and use the same notations as in Prop.~\ref{prop: spKSD validity}.

\begin{proposition}[Asymptotic distributions of spKSD]
\label{prop: asymptotic distribution of spKSD}
    Suppose the assumptions in Prop.~\ref{prop: spKSD validity} hold, and further assume $\bbE_{w, w' \sim R_Q}[ \tilde{u}_P(w, w')^2 ] < \infty$.  
    Let $\{ z_j \}_{j \geq 1}$ be independent draws from $\mathcal{N}(0, 1)$ and denote by $\{c_j\}_{j \geq 1}$ the eigenvalues of $\tilde{u}_{P}$ under $R_Q$, i.e., the solutions of $c_j \phi_j(\cdot) = \bbE_{w \sim R_Q}[ \tilde{u}_P(\cdot, w) \phi_j(w) ]$ for non-zero $\phi_j$. As $n \to \infty$,
    \begin{proplist}
        \vspace{-0.2cm}
        \item Under $H_0: Q = P$, we have $n \ksdHat_{P, \calS} \to_d \sum_{j = 1}^\infty c_j (z_j^2 - 1)$.
        \item Under $H_1: Q \neq P$, we have $\sigma_u^2 \coloneqq 4 \mathrm{Var}_{w \sim R_Q}( \mathbb{E}_{w' \sim R_Q}[\tilde{u}_{P}(w, w') ] ) > 0$, and $\sqrt{n} (\ksdHat_{P, \calS} - \ksd( Q, P; \calS)) \to_d \mathcal{N}(0, \sigma_u^2)$.
    \end{proplist}
\end{proposition}
Prop.~\ref{prop: asymptotic distribution of spKSD} assumes Q also admits a Lebesgue density; when it does not, the stated results still hold true if we additionally assume \emph{i)} the conditions on $Q$ in Prop.~\ref{prop: spKSD validity} for KSD to separate probability measures, and \emph{ii)} $R_Q(A) > 0$ whenever $R_P(A) > 0$ for any measurable set $A \subset \calX^S$.

Similarly to the case with the standard KSD, the cumulative density function of the limiting distribution under $H_0$ has no closed-form expression, but the same bootstrap technique can be employed to estimate the $p$-value using the perturbed samples. The complete algorithm of goodness-of-fit testing with pKSD is given in Algorithm~\ref{alg: GoF with pKSD}.

\section{A Transition Kernel for Multi-Modal Alternatives}
\label{sec:transition kernel}
We consider transition kernels of the Metropolis-Hastings (MH) type \citep{metropolis1953equation, hastings1970monte}. At a current state $x$, a new state $x'$ is proposed by first generating a $d_u$-dimensional random vector $u$ from some known density $g$, then mapping to $x' = h(x | u)$, where $h(\cdot | u)$ is some deterministic, invertible function that is differentiable with differentiable inverse. The proposed state $x'$ is hence a  deterministic function given $x$ and $u$.

We choose in this paper a density $g$ defined on some discrete space $\calU$. The transition kernel is
\begin{talign*}
    \calK(x, A) = \sum_{u \in \calU} \delta_{x'}(A) g( u) \alpha(x, x') + \delta_x(A) r(x) , 
\end{talign*}
where $x' = h(x | u)$ is the proposed state, $\alpha(x, x')$ is an accept-reject rule that guarantees $P$-invariance, $\delta_x(A) = 1$ if $x \in A$ and 0 otherwise, and $r(x) = 1 - \sum_{ u \in \calU} g(u) \alpha(x, x')$. The accept-reject rule $\alpha(x, x')$ is designed to satisfy the \emph{detailed balance condition}:
\begin{talign}
    &\int_{x \in A} \sum_{u \in \calU} \delta_{x'}(B) p(x) g(u) \alpha(x, x') dx \nonumber \\
    &=  \int_{x'\in B} \sum_{u' \in \calU} \delta_{x}(A) p(x') g(u') \alpha(x', x) dx' ,
    \label{eq: detailed balance}
\end{talign}
for all $A, B \in \mathcal{B}(\calX)$. One valid choice is
\begin{talign}
    \alpha(x, x') 
    = 
    \min\left(1, \frac{p(x')g(u')} { p(x) g(u)} \left| \frac{\partial h(x | u)}{\partial x} \right| \right) ,
    \label{eq: mh rule}
\end{talign}
if $x' = h(x | u)$ and $x = h^{-1}(x' | u')$ for some $u, u' \in \calU$, and zero otherwise. Here, $\partial h(x | u) / \partial x$ denotes the Jacobian of the transformation from $x$ to $x'$. Appendix~\ref{appendix: balancing equation} proves that $\alpha(x, x')$ indeed satisfies (\ref{eq: detailed balance}). The accept-reject rule (\ref{eq: mh rule}) resembles those used in Reversible-Jump MCMC \citep{green1995reversible,green2009reversible} and generalises the well-known MH rule, for which the determinant of the Jacobian is 1. 

\subsection{Choosing the Proposal Density}
\label{subsec: choosing the proposal density}
We propose a jump proposal $h(x | u)$ that superposes masses at each mode of $p$. Our choice is motivated by Markov kernels used in the optimisation-based MCMC literature, specifically the \emph{deterministic jumps} proposal in \citet{pompe2020framework}. New states are proposed by randomly selecting a mapping from a set of candidates that are constructed using the location and geometry of the modes of $p$. The resulting kernel is \emph{not} irreducible, so the limiting distribution is not necessarily $P$. Non-irreducibility is essential for the proposed test to work since, under the alternative, the transition kernel should perturb $Q$ to some other distribution for which the KSD between $P$ and the perturbed distribution becomes larger compared with the KSD with the un-perturbed one. This is in contrast to MCMC, which requires irreducibility so that asymptotically the chain can sample from the target distribution.

Denote by $\mu_1, \ldots, \mu_M \in \bbR^d$ the modes  of the density $p$, and $A_1, \ldots, A_M \in \bbR^{d \times d}$ the \emph{inverse} of the Hessian matrices at those points; how to estimate these quantities will be discussed later. When $p$ is a mixture of elliptic distributions such as Gaussian or multivariate $t$-distributions, each $A_m $ can be viewed as the covariance matrix of a component. When the Hessians do not exist (e.g., $-\log p$ is not twice differentiable), we can set $A_m = I_d$ and the remaining discussion still follows.

For a current state $x$, our proposal randomly selects a pair of modes and attempts to map $x$ from one mode to the ``corresponding'' point $x'$ in the other. Formally, let $u = (u_1, u_2) \sim \textrm{Unif}(\{(i, j): 1 \leq i \neq j \leq M \})$ be a uniform random vector over the index set of all $M(M - 1)$ pairs of \emph{distinct} (and ordered) modes, i.e., $g(u) = 1/(M(M-1))$ for all $u$. Given a fixed constant $\theta > 0$, the proposal map is
\begin{talign*}
    \vspace{-0.5em}
    h(x | u)
    = h_\theta(x | u)
    = A_{u_2}^{1/2} A^{-1/2}_{u_1}( x - \theta \mu_{u_1}) + \theta \mu_{u_2} \;,
    \vspace{-0.5em}
\end{talign*}
with the inverse map $h^{-1}(x' | u) = A_{u_1}^{1/2} A^{-1/2}_{u_2}( x - \theta \mu_{u_2}) + \theta \mu_{u_1}$. Intuitively, $h$ sends points from mode $\mu_{u_1}$ to $\mu_{u_2}$ allowing for scaling by local Hessians, and $h^{-1}$ performs the opposite operation. The constant $\theta$ is a hyperparameter introduced to control the scale of the jump, which can increase the ability to detect discrepancies in the mixing weights. Herein, we call $\theta$ the \emph{jump scale}.

Given a current state, our proposal chooses two modes randomly, so a proposed state can potentially lie in a low-density region, thus leading to a low acceptance probability. \citet{pompe2020framework} address this by recording an auxiliary variable for the mode index and augmenting the state space to $\calX \times \{ 1, 2, \ldots, M(M-1)\}$, so that at every step the new state is \emph{guaranteed} to lie near a mode. However, the same trick cannot be used in our case because the augmented density no longer has a well-defined score function. 

\subsection{Understanding the Source of Test Power}
\label{sec: power intuition}
\begin{figure}
    \centering
    \includegraphics[width=0.48\textwidth]{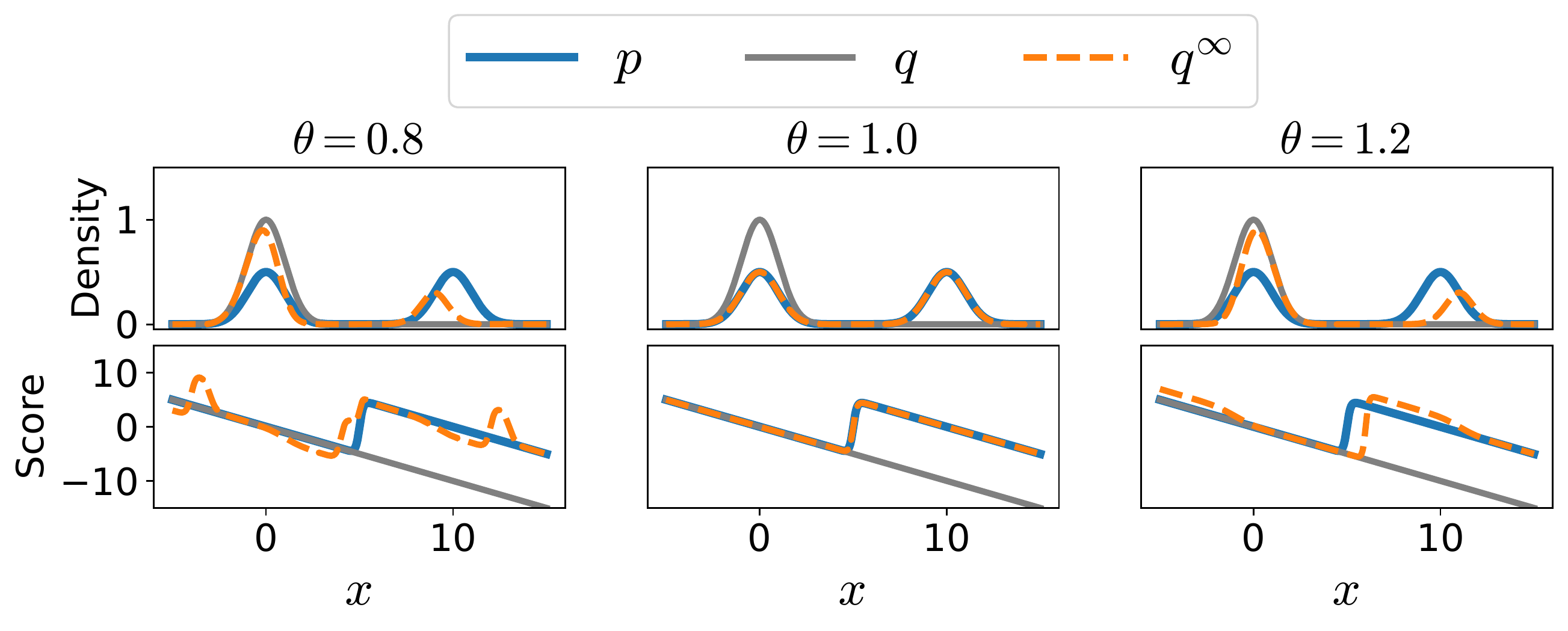}
    \includegraphics[width=0.4\textwidth]{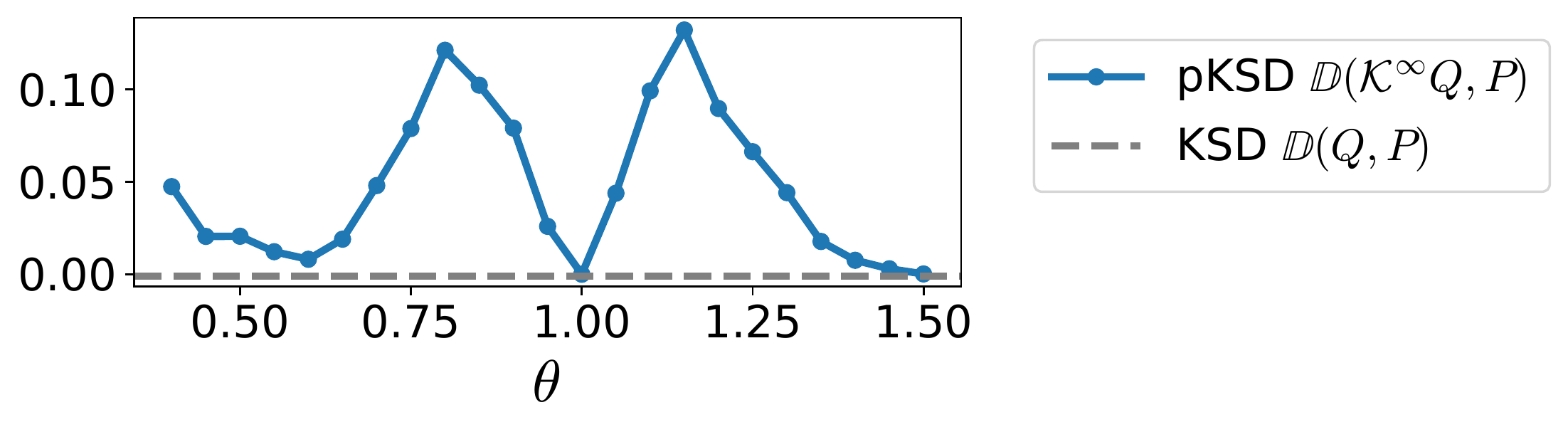}
    \vspace{-0.4cm}
    \caption{Top: Densities and score functions of $p, q$ and the limiting distribution $q^\infty$ in (\ref{eq: limiting distribution}). Bottom: pKSD with different jump scales $\theta$, compared with KSD.}
    \vspace{-0.5cm}
    \label{fig: transition kernel interpretation}
\end{figure}
To understand the improvement in test power against multi-modal alternatives, we characterise the limiting distribution of a general distribution $Q$ with a positive density $q$ when we apply the perturbation with infinitely many steps (i.e., $\calK^T$ with $T = \infty$). For simplicity, we assume $M = 2$ and $A_1 = A_2 = I_d$ are identity matrices, so that the proposal function is $h_\theta(x | u) = x - \theta (\mu_{u_1} - \mu_{u_2})$ for $x \in \calX$ and $u = (u_1, u_2) \in \calU = \{ (i, j): 1 \leq i \neq j \leq 2 \}$. Thus,  given a current state $x$, the transition kernel proposes moves to  $ x + \theta(\mu_1 - \mu_2)$ and $x - \theta(\mu_1 - \mu_2) $ with equal probability.

\begin{proposition}
\label{prop: limiting distribution}
    Under the assumptions of Sec.~\ref{sec: power intuition}, the limiting distribution under $\calK$ with the initial distribution $Q$ is $(\calK^\infty Q) (A) = \int_{x \in A} q^{\infty}(x) dx$, $A \in \calB(\calX)$, where
    \begin{talign}
        q^\infty (x) 
        \coloneqq p(x) \frac{\sum_{s \in \bbZ} q(x + s\nu)}{ \sum_{k \in \bbZ} p(x + k\nu) } \;,
        \label{eq: limiting distribution}
    \end{talign}
    and $\nu \coloneqq \theta(\mu_1 - \mu_2)$.
\end{proposition}
A proof is in Appendix~\ref{pf: limiting distribution}. Prop.~\ref{prop: limiting distribution}  shows that the limiting density under $\calK$ is the target density $p$ weighted by the ratio between the total masses of $q$ and $p$ over a discrete grid. 

To understand why this helps to increase the KSD value, we first rewrite KSD as
\begin{talign*}
    \ksd(Q, P) = \bbE_{x, x' \sim Q}[ \delta_{q, p}(x)^\top k(x, x') \delta_{q, p}(x') ] \;,
\end{talign*}
where $\delta_{q, p}(x) \coloneqq s_q(x) - s_p(x)$ is the score difference. This holds whenever $k$ is an integrally strictly positive definite kernel \citep[Thm.~3.6]{liu2016kernelized}. The pKSD is then
\begin{talign*}
    \ksd(Q, P; \calK^\infty)
    = \bbE_{x, x' \sim Q^\infty}[ \delta_{q^\infty, p}(x)^\top k(x, x') \delta_{q^\infty, p}(x') ] 
    \;,
\end{talign*}
where, by Prop.~\ref{prop: limiting distribution}, the score difference becomes $\delta_{q^\infty, p}(x) = s_{q^\infty}(x) - s_p(x) = \nabla \log \phi_q(x) - \nabla \log \phi_p(x)$, where $\phi_q(x) \coloneqq \sum_{s \in \bbZ} q(x + s\nu)$ and similarly for $\phi_p$. 

The operator $\phi$ superposes densities along a grid, thus allowing to create local discrepancy in the high-probability regions of $Q$, for example by exchanging masses between modes. 

As a concrete example, we consider the setup in Thm.~\ref{thm: multimodal}, where $Q = \calN(0, I_d)$, and $P = \pi \calN(0, I_d) + (1-\pi) \calN(\Delta, I_d)$ for some $\pi \in (0, 1)$ and $\Delta \in \bbR^d$. The operator has created discrepancy in high-probability regions of $Q$, as demonstrated in Fig.~\ref{fig: transition kernel interpretation}. This also highlights the role of $\nu$: when $\nu = \Delta$, the two components will overlap almost exactly under the perturbation, so $\delta_{q^\infty, p}(x) \approx 0$ near $x = 0$, and the KSD will remain small (Fig.~\ref{fig: transition kernel interpretation}). It is hence crucial to tune $\nu$ (equivalently, $\theta$). One can in principle select $\theta$ by maximising the (approximate) test power, similarly to the idea in \citet{jitkrittum2017linear}. However, gradient-based approaches are infeasible as pKSD is \emph{not} differentiable with respect to $\theta$. An alternative is to use grid-search over some finite set of $\theta$ values.

\subsection{Choosing the Set of Perturbations $\calS$}
\label{subsec: spKSD, ospKSD and Implementation Details}
It remains to choose the set of perturbations $\calS$ in spKSD. We propose two ways to construct $\calS$, one based on grid-search, and the other based on optimisation. 

For the grid-based approach, we choose a set of values $\{ \theta_s \}_{s=1}^{S-1}$ and let $\calS = \{\calK_{\textrm{id}}, \calK_1, \ldots, \calK_{S-1} \}$, where $\calK_s$ is the transition kernel described in this section with jump scale $\theta_s$. We propose to choose each $\theta_s $ close to 1, following the observations in Sec.~\ref{sec: power intuition}. We still refer to the resulting divergence as spKSD.

For the optimisation-based approach, we use only two transition kernels $\calS = \{\calK_{\textrm{id}}, \calK_\theta \}$, where $\calK_\theta$ has jump scale $\theta$ that is tuned by maximising a proxy for the asymptotic test power. Due to the asymptotic normality proved in Prop.~\ref{prop: asymptotic distribution of spKSD}, we can adopt the same approach in \citet[Prop.~4]{jitkrittum2017linear} to approximate the asymptotic power with the ratio
\begin{talign}
    \ksdHat_{P, \calK_{\theta}^{T}} \; / \; \hat{\sigma}_u
    \;,
    \label{eq: test power proxy}
\end{talign}
where $\hat{\sigma}_u$ is an estimate of the asymptotic standard deviation $\sigma_u$ is given by the square root of 
\begin{align*}
    \hat{\sigma}_u^2
    \coloneqq
     \frac{4}{n^3} \sum_{i=1}^n \left( \sum_{j=1}^n H_{i, j} \right)^2
     - \frac{4}{n^4} \left( \sum_{i, j = 1}^{n} H_{i, j} \right)^2 
     \;,
\end{align*} 
with $H_{i, j} \coloneqq u_{P}(x_i, x_j) + u_P(x_i^\theta, x_j^\theta)$ and $x_i^\theta \sim \calK^{T}_\theta Q$; see also \citet[Eq.~8]{schrab2022ksd}. Since the objective (\ref{eq: test power proxy}) (in particular, $\calK_\theta$) is not differentiable with respect to $\theta$, we still choose $\theta$ from a pre-specified finite set $\{ \theta_s \}_{s=1}^S$. The objective is hence $\max_{\theta \in \{ \theta_1, \ldots, \theta_S\}} \ksdHat_{P, \calK_\theta^{T}} / \hat{\sigma}_u $. We call the resulting discrepancy the \emph{optimised sum-pKSD} (ospKSD).

Whether the grid-based or the optimisation-based method should be preferred requires trade-offs and depends on the specific problem at hand --- The spKSD requires no held-out sets, but can suffer from a low test power if $\{\theta_s\}_{s=1}^{S-1}$ is poorly chosen in that most of $\calK_s$ fail to improve the test power. On the other hand, ospKSD uses a judiciously tuned $\theta$, but the data-splitting can also lead to a drop in test power. In our experiments, we find that spKSD tends to work better for target distributions with a simple geometry, specifically mixtures of \emph{elliptic} distributions \citep{cambanis1981theory} (e.g., the Gaussian mixture examples). However, for distributions whose mixing components have non-elliptic contours (e.g., the mixture of $t$ and banana example, and the sensor network localisation example), the benefit of optimisation seems to overweigh the negative impact due to data-splitting, and ospKSD outperforms spKSD.

\subsection{Estimating Mode Vectors and Hessians}
We estimate $\mu_j$ and $A_j$ by the local minima and Hessians of $- \log p$. To do so, we run in parallel a sequence of BFGS optimisers \citep{nocedal2009numerical} initiated at different starting points, following \citet{pompe2020framework}. BFGS is used because it returns both the local optima and approximated Hessians at those points. The optima are then merged if their weighted Mahalanobis distance is smaller than a pre-specified threshold. In our experiments, we initialise the optimisers from a set of size $n_{\textrm{init}}$, constructed either by sampling uniformly from a hyper cube $[L_1, U_1] \times \cdots \times [L_d, U_d]$ (for spKSD), or by using both randomly sampled data and some training set (for ospKSD). The full procedure is described in Appendix~\ref{subsec: finding the mode vectors via optimisation}.

\section{Related Work}
\label{sec: related works}
\paragraph{Perturbation with convolution}
The idea of combining a discrepancy with perturbation has been widely studied, where the perturbation is often a convolution operator with Gaussian noise. E.g., the \emph{spread divergence} \citep{zhang2020spread} combines Gaussian convolution with Kullback-Leibler (KL) divergence (more generally, $f$-divergences) to solve the issue that KL divergence is ill-defined when the distributions have undefined densities or unmatched support. In generative modelling, \emph{denoising score matching} \citep{vincent2011connection} and \emph{Noise Conditional Score Networks} \citep{song2019generative} combine Gaussian convolution with score matching to improve computational efficiency or estimation quality. Notably, convolution is \emph{not} invariant to the target distribution, thus rendering the score function intractable. This is why we chose a MH-type kernel instead.

\paragraph{Perturbation with convex combination}
In score matching, \citet{zhang2022towards} addresses the blindness of Fisher Divergence by mapping the target and candidate distributions to a convex combination with a Gaussian distribution, thereby ``connecting'' the well-separated modes. A similar idea cannot be applied to improve the KSD test, as, similarly to convolution, the resulting target distribution no longer has a tractable score function.

\paragraph{Perturbation with annealing}
Another choice of perturbation is to anneal both distributions by raising the densities to some power, which is studied in \citet{wenliang2020blindness}. Although the score function remains tractable under this perturbation, annealing alone cannot solve the blindness of score-based discrepancies, as noted by the authors and in \citet{zhang2020spread}. Moreover, sampling from the annealed candidate distribution is also non-trivial. 


\section{Experiments}
\label{sec: experiments}
\begin{figure}
    \centering
    \begin{minipage}[t]{.48\textwidth}
         \centering
         \includegraphics[width=\textwidth]{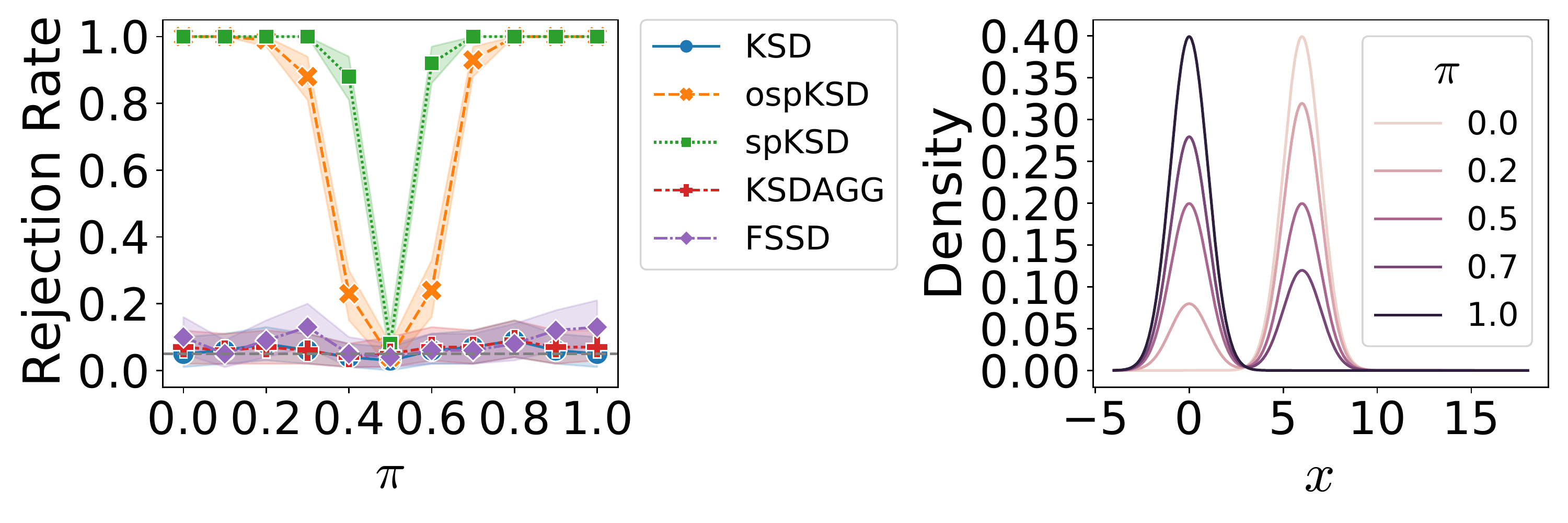}
         \vspace{-0.8cm}
         \caption{One-dimensional Gaussian mixture example. Samples are drawn with a different mixing weight $\pi$.}
        \label{fig: bimodal ratio_s}
    \end{minipage}
    \vspace{-0.2cm}
\end{figure}

\begin{figure}
    \centering
    \includegraphics[width=.20\textwidth]{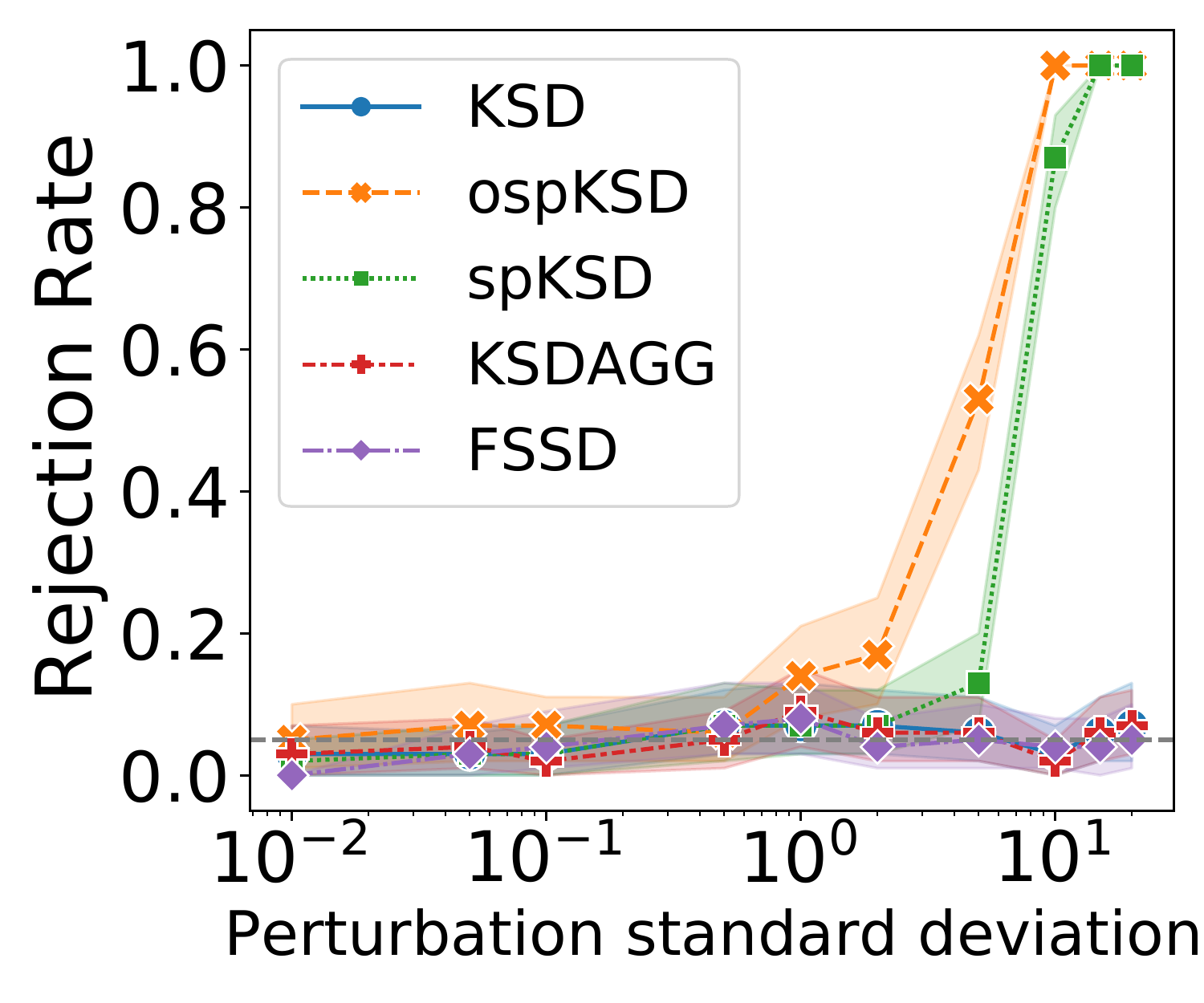}
    \vspace{-0.2cm}
    \caption{Mixture of 10 $t$ and 10 banana distributions.}
    \label{fig: tbanana noise}
    \vspace{-0.5cm}
\end{figure}

We use $51$ jump scales $\theta$ equally spaced in $[0.5, 1.5]$, a heuristic that we find works well in practice. All samples have size $n=1000$. We compare the ospKSD and spKSD tests against benchmarks including KSD test and two variants (\textsc{KSDAgg} and FSSD). All experiments are run with level $\alpha = 0.05$ using the IMQ kernel $k(x, y) = ( 1 + \|x - y \|_2^2 / \lambda)^{-1/2}$, where $\lambda$ is chosen to be $\texttt{median}_{i<j} \{\| x_i - x_j \|_2^2\}$. \textsc{KSDAgg} follows the setup in \citet{schrab2022ksd}, and FSSD follows \citet{jitkrittum2017linear}. The probability of rejecting the null hypothesis is estimated by averaging the test output over 100 repetitions, except in the sensors location example, which is repeated 10 times. Translucent shades represent $95\%$-CIs. The number of transition steps $T$ is selected to be $10$ for the Gaussian mixture example, $100$ for the mixture of $t$ and banana distributions example, and $1000$ for the sensor network localisation. 

Discussions on how to choose $T$ in practice, as well as supplementary plots and experiments, are held in Appendix~\ref{appendix: details on experiments}. In particular, in Appendix~\ref{appendix: GB-RBM}, we include an additional experiment concerning a 50-dimensional Gaussian-Bernoulli Restricted Boltzmann Machine (RBM) \cite{cho2013gaussian}, a latent variable model that can be viewed as a mixture of Gaussian distributions. Code for reproducing all experiments can be found at \url{github.com/XingLLiu/pksd}.

\paragraph{Gaussian mixture}
The target has density $p(x) \propto  \exp\left(- \frac{1}{2} x^2 \right) + 0.5 \exp\left(- \frac{1}{2} (x - 6)^2 \right)$. Samples are drawn with a different mixing weight $\pi \in [0, 1]$ of the first component. The results are presented in Fig.~\ref{fig: bimodal ratio_s}. KSD, \textsc{KSDAgg} and FSSD all have a power close to the level $0.05$ regardless of the value of $\pi$, which is not surprising due to the blindness of KSD. In comparison, ospKSD and spKSD achieve almost perfect power when $\pi$ deviates from the true value 0.5. Fig.~\ref{fg: bimodal level and power} in the Appendix verifies numerically that ospKSD and spKSD achieve the desired level under $H_0$.

\paragraph{Mixture of $t$ and banana distributions}
We consider a mixture of 10 multivariate $t$-distributions and 10 banana-shaped distributions with $t$-tails in $d=50$ dimensions, also studied in \citet{pompe2020framework}. Each component has an equal weight $0.02$ and is centered randomly in $[-20, 20]^d$, giving rise to a target with sparsely located, non-elliptic modes. Samples are drawn with a different set of weights $\{w_j\}_{j=1}^{20}$ formed by sampling $\tilde{w}_j \sim \calN(0, \sigma_s^2)$, $\sigma_s > 0$, and normalising $w_j \propto \exp(\tilde{w}_j)$. Other details are held in Appendix~\ref{appendix: t banana}. Fig.~\ref{fig: tbanana noise} (left) shows the results. As $\sigma_s$ increases, the weights in the two distributions deviate further, so ospKSD and spKSD achieve a higher power. All the others perform poorly for all values of $\sigma_s$, because the components have almost no overlapping high-density regions.

\paragraph{Sensor network localisation}
\citet{tak2018repelling} use Bayesian methods to infer the locations of sensors from noisy distance data. This is a modification of the example in \citet{ihler2005nonparametric} that has been used as a benchmark for MCMC samplers designed for multi-modal distributions \citep{pompe2020framework, ahn2013distributed, lan2014wormhole}. Here, six sensors $x_1, \ldots, x_6$ are located in $[0, 1]^2$, four of which have unknown locations and the remaining two are known. We observe distance $y_{ij}$ between two sensors $x_i, x_j$ with probability $\exp(- \| x_i - x_j \|_2^2 / (2 \times 0.3^2))$. If observed, the distance follows a Gaussian distribution $y_{ij} \sim \calN(\| x_i - x_j\|, 0.02^2)$. Full model details are held in Appendix~\ref{appendix: sensors}. To draw posterior samples, \citet{tak2018repelling} propose the \emph{repelling-attracting Metropolis} (RAM), which is an MCMC algorithm designed for efficient learning of multi-modal target distributions. RAM relies on a Gaussian proposal with a fixed covariance matrix $\sigma^2 I_d$ to propose new states in its uphill and downhill steps. The \emph{scale} $\sigma$ needs to be tuned to facilitate transitions between modes.

\begin{table}[t]
    \centering
    \caption{GOF tests for checking the quality of RAM samples with different scales. Reported values are the number of rejections over 10 repetitions with level $0.05$.}
    \vspace{-0.1cm}
    \begin{adjustbox}{max width=.45\textwidth} 
        \begin{tabular}{ccccccc}
            \toprule
            Methods & & KSD & ospKSD & spKSD & \textsc{KSDAgg} & FSSD \\
            \midrule
            \multirow{3}{*}{RAM scale $\sigma$}
             & 0.1  & 0 & 8 & 0 & 0 & 0 \\
             & 0.5  & 0 & 0 & 0 & 0 & 1 \\
             & 1.08 & 10 & 8 & 10 & 10 & 7 \\
            \bottomrule
        \end{tabular}
        \label{tab: sensors}
    \end{adjustbox}
\end{table}

\begin{figure}[t]
    \centering
    \vspace{-0.4cm}
    \includegraphics[width=1.\linewidth]{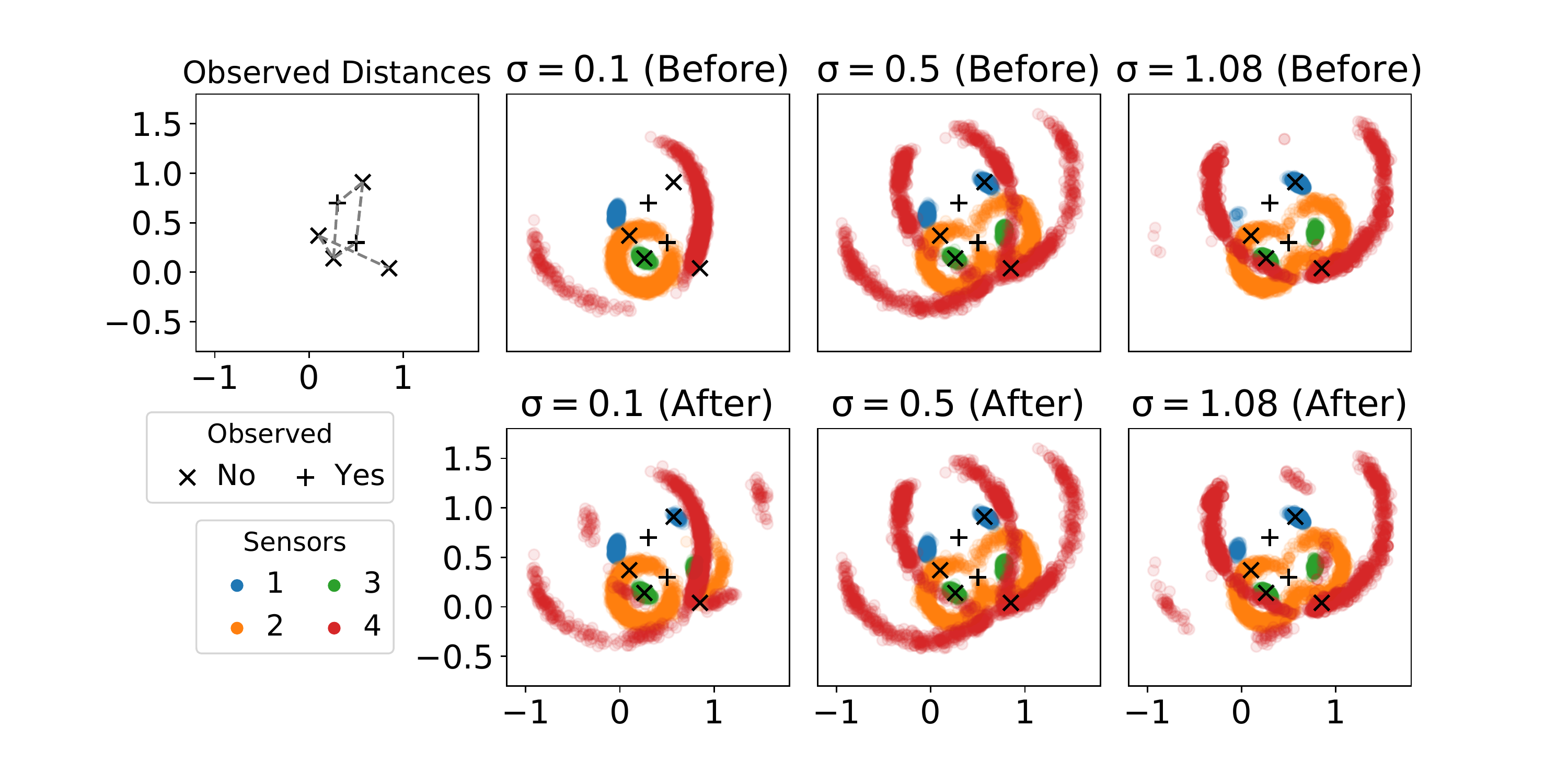}
    \vspace{-0.8cm}
    \caption{True and inferred locations of sensors. Black plus signs mark the location of the unobserved sensors, and black crosses indicate the location of the observed ones.}
    \vspace{-0.5cm}
    \label{fig: sensors}
\end{figure}

We run RAM with different scales $\sigma$, each for $420,000$ iterations. We discard the first $20,000$ particles as burn-in and thin the remaining to obtain a sample of size $n=4000$. We then evaluate the quality of the samples by applying a GOF test. Fig.~\ref{fig: sensors} shows the posterior samples generated using each $\sigma$. The samples from $\sigma = 0.5$ seem to capture all the modes of the posterior and is consistent with the results reported in \citet[Fig.~5]{tak2018repelling}, whereas samples from $\sigma = 0.1$ and $1.08$ clearly miss some modes. We then compare the test results in Table~\ref{tab: sensors}, which reports the number of rejections over 10 repetitions. All tests reject most runs for $\sigma = 1.08$ (the scale used in \citet{tak2018repelling}) and almost no run for $\sigma = 1.08$, which is consistent with the posterior plots. However, for $\sigma = 0.5$, no method except ospKSD rejected the null hypothesis, demonstrating again the ability of ospKSD to detect missing modes. Fig.~\ref{fig: sensors} also shows the particles \emph{after} perturbation by the (non-identity) transition kernel used by ospKSD, from which some particles seem to have moved to the missing high-density regions. spKSD in this case also performed poorly with no sample for $\sigma = 0.1$ being rejected, which is potentially because the benefit of not having a held-out set outweighs that of using a tuned transition kernel for this example.

\section{Discussion and Conclusion}
\label{sec: conclusion}
We show with a bimodal Gaussian example that GOF tests based on KSD can fail when the target has well-separated modes. To increase its power, we propose to perturb both the candidate and the target distributions using a Markov process before applying the KSD test. Empirical results suggest that our methods (ospKSD and spKSD) are more sensitive to discrepancies in the mixing weights of multimodal distributions, and can achieve remarkably high power particularly when the  mixing components are elliptic distributions.


\subsection{Limitations and Future Work}
The ospKSD and spKSD rely heavily on accurate estimation of the mode locations and Hessians, which can be extremely challenging and computationally costly for high-dimensional problems. Additionally, the jump proposal of the transition kernel used in the proposed methods is constructed specifically for targets that are mixtures of elliptic distributions, which may be inappropriate for targets with more complicated geometrical structure. Further investigations could aim to find perturbation operators that scale better with dimensionality or that suit a wider family of target distributions. 

Moreover, both spKSD and ospKSD require careful hyper-parameter setting. The spKSD, as a sum-like statistic, requires a trade-off between the test power and the number of elements in the grid $\calS$. Although the heuristic described in Section~\ref{sec: experiments} is found to perform decently in our experiments, it is of practical interest to analyse the sensitivity of the test performance to the grid size both empirically and theoretically. The ospKSD, on the other hand, requires a held-out dataset to tune $\theta$, potentially reducing test power due to data-splitting. One possible approach to mitigate this problem is to combine ospKSD with the aggregated testing framework described in \citet{schrab2022ksd} to avoid splitting the data.

\section*{Acknowledgements}
XL is supported by the President’s PhD Scholarships of Imperial College London and the EPSRC StatML CDT programme EP/S023151/1. ABD is supported by Wave 1 of The UKRI Strategic Priorities Fund under the EPSRC Grant EP/T001569/1 and EPSRC Grant EP/W006022/1, particularly the “Ecosystems of Digital Twins” theme within those grants \& The Alan Turing Institute. We thank the anonymous reviewers for their valuable comments.


\bibliography{ref}
\bibliographystyle{icml2023}

\newpage
\appendix
\onecolumn
\section{Proof of Proposition~\ref{prop: distribution under H1 converges to null distribution}}
\label{proof: distribution under H1 converges to null distribution}

Fixing positive integer $\nu$, we can write $n_\nu \ksdHat_{P_{\nu}} = n_\nu \ksdHat_{Q} + n_\nu( \ksdHat_{P_{\nu}} - \ksdHat_{Q} )$. Under the stated assumptions on the kernel $k$ and that $\bbE_{x, x'\sim Q}[ u_Q(x, x')^2 ] < \infty$, we can apply \citet[Thm~4.1]{liu2016kernelized} to conclude that, as $\nu \to \infty$,
\begin{talign*}
    n_\nu \ksdHat_{Q} \to_d \sum_{j = 1}^\infty c_j (z_j^2 - 1) \ ,
\end{talign*}
where $z_j, c_j$ are as defined in Prop.~\ref{prop: distribution under H1 converges to null distribution}. If we could furthermore show that $n_\nu( \ksdHat_{P_{\nu}} - \ksdHat_{Q}) \to 0$ in probability as $\nu \to \infty$, then the desired result would follow from Slutsky's Theorem (see, e.g., \citet{Casella2001statistial}).

To prove the convergence in probability, we fix $\epsilon > 0$ and denote by $\textrm{Pr}_Q$ the probability under $Q$. We also omit the dependence of $n$ on $\nu$ for brevity. The Markov inequality yields
\begin{talign}
    & \textrm{Pr}_Q( n| \ksdHat_{P_\nu} - \ksdHat_{Q} | \geq \epsilon ) \nonumber \\
    &\leq \frac{n}{\epsilon} \bbE_{x_1, \ldots, x_n \sim Q}[  | \ksdHat_{P_\nu} - \ksdHat_{Q} | ] \nonumber\\
    &= \frac{n}{\epsilon} \bbE_{x_1, \ldots, x_n \sim Q} \left| \frac{1}{n(n - 1)} \sum_{1 \leq i \neq j \leq n} u_{P_\nu}(x_i, x_j) - u_Q(x_i, x_j) \right| \nonumber\\
    &\leq \frac{n}{\epsilon} \frac{1}{n(n - 1)} \sum_{1 \leq i \neq j \leq n} \bbE_{x_i, x_j \sim Q} | u_{P_\nu}(x_i, x_j) - u_Q(x_i, x_j) | \nonumber\\
    &= \frac{n}{\epsilon} \bbE_{x, x' \sim Q} | u_{P_\nu}(x, x') - u_Q(x, x') | \nonumber\\
    &\leq \frac{n}{\epsilon} \{ 
        \bbE_{x, x' \sim Q} | s_{p_\nu}(x)^\top s_{p_\nu}(x') - s_q(x)^\top s_q(x') | | k(x, x') |  \nonumber\\
        &\quad + \bbE_{x, x' \sim Q} | (s_{p_\nu}(x) - s_q(x))^\top \nabla_{x'} k(x, x') | \nonumber\\
        &\quad + \bbE_{x, x' \sim Q} | (s_{p_\nu}(x') - s_q(x'))^\top \nabla_x k(x, x') |
    \} \nonumber\\
    &\leq \frac{n}{\epsilon} \Big\{ 
        \left( \bbE_{x, x' \sim Q} \left[ ( s_{p_\nu}(x)^\top s_{p_\nu}(x') - s_q(x)^\top s_q(x') )^2 \right] \right)^{1/2}  
        \left(\bbE_{x, x' \sim Q} \left[ k(x, x')^2 \right] \right)^{1/2} \nonumber\\
        &\qquad\quad + \left( \bbE_{x \sim Q} \left[ \| s_{p_\nu}(x) - s_q(x) \|_2^2 \right] \right)^{1/2} 
        \left( \bbE_{x, x' \sim Q} \left[\| \nabla_{x'} k(x, x') \|_2^2 \right] \right)^{1/2} \nonumber\\
        &\qquad\quad + \left( \bbE_{x \sim Q} \left[ \| s_{p_\nu}(x) - s_q(x) \|_2^2 \right] \right)^{1/2} 
        \left( \bbE_{x, x' \sim Q} \left[\| \nabla_{x} k(x, x') \|_2^2 \right] \right)^{1/2} \label{eq: three terms}
    \Big\} .
\end{talign}
We bound each of the three terms individually. For the first term, we have
\begin{talign*}
    & \bbE_{x, x' \sim Q} \left[ \left( s_{p_\nu}(x)^\top s_{p_\nu}(x') - s_q(x)^\top s_q(x') \right)^2 \right] \\
    &= \bbE_{x, x' \sim Q} \left[ \left( s_{p_\nu}(x)^\top (s_{p_\nu}(x') - s_q(x')) + (s_{p_\nu}(x) - s_q(x))^\top s_q(x') \right)^2 \right] \\
    &\leq 2 \underbrace{\bbE_{x, x' \sim Q} \left[ \left( s_{p_\nu}(x)^\top (s_{p_\nu}(x') - s_q(x')) \right)^2 \right]}_{\eqqcolon T_1}
    + 2 \underbrace{\bbE_{x, x' \sim Q} \left[ \left( (s_{p_\nu}(x) - s_q(x))^\top s_q(x') \right)^2 \right]}_{\eqqcolon T_2} ,
\end{talign*}
where the last line follows from the fact that $(a + b)^2 \leq 2a^2 + 2b^2$ for any $a, b \in \bbR$. Now, applying the Cauchy-Schwarz inequality gives
\begin{talign*}
    T_1
    &\leq 2 \bbE_{x, x' \sim Q} [ \| s_{p_\nu}(x) \|_2^2 \| s_{p_\nu}(x') - s_q(x') \|_2^2  ] \\
    &\leq 2 \left( 2\bbE_{x \sim Q}[ \| s_{p_\nu}(x) - s_q(x) \|_2^2 ] + 2\bbE_{x \sim Q}[ \| s_q(x)\|_2^2 ] \right) \bbE_{x'\sim Q}[ \| s_{p_\nu}(x') - s_q(x') \|_2^2  ] \\
    &\leq 
    4 F_\nu^2 + 4 \bbE_{x \sim Q}[ \| s_q(x) \|_2^2 ] F_\nu,
\end{talign*}
where $F_\nu \coloneqq \bbE_{x \sim Q}[ \| s_{p_\nu}(x) - s_q(x)\|_2^2 ]$ is the Fisher Divergence between $P_\nu$ and $Q$, and where the second line holds because 
\begin{talign*}
    \| s_{p_\nu}(x) \|_2^2 = \| s_{p_\nu}(x) - s_q(x) + s_q(x) \|_2^2 \leq 2 \|s_{p_\nu}(x) - s_q(x)\|_2^2 + 2\| s_q(x) \|_2^2 .
\end{talign*}
A similar argument by Cauchy-Schwarz inequality shows that 
\begin{talign*}
    T_2
    \leq 
    \bbE_{x \sim Q}[ \| s_q(x) \|_2^2 ] F_\nu \;.
\end{talign*}
Combining the bounds for $T_1$ and $T_2$ yields
\begin{talign*}
    \bbE_{x, x' \sim Q} [ ( s_{p_\nu}(x)^\top s_{p_\nu}(x') - s_q(x)^\top s_q(x') )^2]
    &\leq 
    8 F_\nu^2 + 10 \bbE_{x \sim Q}[ \| s_q(x) \|_2^2 ] F_\nu \;.
\end{talign*}
By the assumed boundedness of the kernel and its gradients, there exists a positive constant $M < \infty$ depending only on $k$ and $Q$ such that
\begin{talign*}
    \max\left\{ 
    \bbE_{x, x' \sim Q}[ | k(x, x') | ], \ 
    \bbE_{x, x' \sim Q}[ \| \nabla_{x'} k(x, x')\|_2^2], \ 
    \bbE_{x, x' \sim Q}[ \| \nabla_{x } k(x, x')\|_2^2]
    \right\} \leq M \;.
\end{talign*}
We hence conclude from (\ref{eq: three terms}) that
\begin{talign}
    \textrm{Pr}_Q( n| \ksdHat_{P_\nu} - \ksdHat_{Q} | \geq \epsilon )
    &\leq
    \frac{n}{\epsilon} \left[ M^{1/2} \left( 8 F_\nu^2 + 10 \bbE_{x \sim Q}[ \| s_q(x) \|_2^2 ] F_\nu \right)^{1/2} + 2 M^{1/2} F_\nu^{1/2} \right] \nonumber \\
    &\leq
    \frac{n}{\epsilon} M^{1/2} \underbrace{\left[ F_\nu^{1/2} \left(8 F_\nu + 10 \bbE_{x \sim Q}[\| s_q(x) \|_2^2] \right)^{1/2} + 2F_\nu^{1/2}  \right]}_{\eqqcolon T_3} \;.
    \label{eq: bound in terms of fisher div}
\end{talign}
The term $T_3$ is $O( \max(F_\nu, F_\nu^{1/2}) )$.
Therefore, if $n = n_\nu = o( 1/\max(F_\nu, F_\nu^{1/2}) )$, then the right hand side of (\ref{eq: bound in terms of fisher div}) converges to 0 as $\nu \to \infty$, thus $n_\nu| \ksdHat_{P_\nu} - \ksdHat_{Q} | \to 0$ in probability. This completes the proof.

\section{Proof of Theorem~\ref{thm: multimodal}}
\label{proof: multimodal}

Theorem~\ref{thm: multimodal} follows directly from Proposition~\ref{prop: distribution under H1 converges to null distribution} and the next lemma, which states that the Fisher Divergence between $Q$ and $P_\Delta$ decays with a rate at least exponentially fast in the inter-modal distance $\| \Delta \|_2$.

\begin{lemma}
\label{lemma: fisher div goes to zero}
    Under the same assumptions in Prop.~\ref{prop: distribution under H1 converges to null distribution}, we have $\bbE_{x \sim Q}[ \| s_{p_{\Delta}}(x) - s_q(x) \|_2^2] = o\left( e^{-\| \Delta \|_2^2 / 32} \right)$ \;.
\end{lemma}

\begin{proof}[Proof of Lemma~\ref{lemma: fisher div goes to zero}]
    For any $\delta > 0$, define $B_\delta \coloneqq \{x \in \bbR^d: \| x \|_2 \leq \delta \} $. We have the following decomposition
    \begin{talign}
        &\bbE_{x\in Q}\left[ \| s_{p_\Delta}(x) - s_q(x) \|_2^2 \right] \nonumber \\
        &= \bbE_{x \sim Q}\left[ \delta_x(B_\delta) \| s_{p_\Delta}(x) - s_q(x) \|_2^2 \right] 
        + \bbE_{x \sim Q}\left[ \delta_x(\bbR^d \backslash B_\delta) \| s_{p_\Delta}(x) - s_q(x) \|_2^2 \right] \;.
        \label{eq: score difference decomposition}
    \end{talign}
    The rest of the proof proceeds with bounding the two terms separately. We first note that standard computation gives
    \begin{talign*}
        \frac{p_\Delta(x)}{q(x)}
        &= \frac{\pi \exp\left( - \frac{1}{2} \|x\|^2 \right) + (1 - \pi) \exp\left( - \frac{1}{2}\|x - \Delta\|^2 \right)}{\exp\left( - \frac{1}{2}\|x\|^2 \right)}
        = \pi + (1 - \pi) \exp\left( \Delta^\top x - \frac{1}{2}\|\Delta\|^2 \right) \;,
    \end{talign*}
    and
    \begin{talign*}
        \| s_{p_\Delta}(x) - s_q(x) \|_2^2
        &= \left\| \frac{(1 - \pi) \Delta \exp\left( \Delta^\top x - \frac{1}{2}\|\Delta\|_2^2 \right) }{\pi + (1 - \pi) \exp\left( \Delta^\top x - \frac{1}{2}\|\Delta\|_2^2 \right)} \right\|_2^2 
        = \frac{(1 - \pi)^2 \| \Delta\|_2^2 }{ \left( 1 - \pi + \pi \exp\left( -\Delta^\top x + \frac{1}{2}\|\Delta\|_2^2 \right) \right)^2} \;.
    \end{talign*}
    
    For $x \in B_\delta$, Cauchy-Schwarz inequality implies $\Delta^\top x \leq \| \Delta\|_2 \| x \|_2 \leq \delta \|\Delta \|_2 $. Hence
    \begin{talign*}
        \| s_{p_\Delta}(x) - s_q(x) \|_2^2
        \leq \frac{(1 - \pi)^2 \|\Delta\|_2^2}{ \pi^2 \exp\left( -2\Delta^\top x + \|\Delta\|_2^2 \right)} 
        \leq \frac{(1-\pi)^2 \|\Delta\|_2^2}{\pi^2 \exp\left( -2\delta \|\Delta\|_2 + \|\Delta\|_2^2 \right)} \;,
    \end{talign*}
    and the first term of (\ref{eq: score difference decomposition}) can be bounded as
    \begin{talign*}
        \bbE_{x \sim Q}\left[ \delta_x(B_\delta) \| s_{p_\Delta}(x) - s_q(x) \|_2^2 \right]
        \leq \bbE_{x \sim Q}\left[ \delta_x(B_\delta) \frac{(1 - \pi)^2 \|\Delta\|_2^2}{ \pi^2 \exp\left( -2 \delta \| \Delta\|_2 + \|\Delta\|_2^2 \right)} \right]
        \leq \frac{(1 - \pi)^2\|\Delta\|_2^2}{ \pi^2 \exp\left( -2\delta \| \Delta \|_2 + \| \Delta \|_2^2 \right)} \;,
    \end{talign*}
    where the last inequality follows from the fact that $\bbE_{x \sim Q}[ \delta_x(B_\delta) ] \leq 1$.
    
    To bound the second term of (\ref{eq: score difference decomposition}), we note that for $x \in \bbR^d \backslash B_\delta$,
    \begin{talign*}
        \| s_{p_\Delta}(x) - s_q(x) \|_2^2
        \leq \frac{(1 - \pi)^2 \| \Delta\|_2^2}{ (1 - \pi)^2 }
        = \| \Delta\|_2^2 \;.
    \end{talign*}
    Therefore, 
    \begin{talign*}
        \bbE_{x \sim Q}\left[ \delta_x(B_\delta) \| s_{p_\Delta}(x) - s_q(x) \|_2^2 \right]
        \leq \| \Delta\|_2^2 \bbE_{x \sim Q}\left[ \delta_x(B_\delta) \right]
        \leq 5^d \| \Delta\|_2^2 \exp\left(- \frac{\delta^2}{8} \right) \;,
    \end{talign*}
    by the tail probability of the norm of centred Gaussian random vectors (see e.g.\ \citet[Prop.~2.5]{wainwright2019high}). Combining these results we have
    \begin{talign*}
        \bbE_{x\in Q}[ \| s_{p_\Delta}(x) - s_q(x) \|_2^2 ]
        &\leq \frac{(1 - \pi)^2 \|\Delta\|_2^2}{ \pi^2 \exp\left( -2\delta \|\Delta\|_2 + \| \Delta \|_2^2 \right)} + 5^d \| \Delta\|_2^2 \exp\left(- \frac{\delta^2}{8} \right)  \\
        &= (1-\pi)^2\pi^{-2} \|\Delta\|_2^2 \exp\left( - \frac{1}{17} \|\Delta\|_2^2 \right) + 5^d \|\Delta\|_2^2 \exp\left( - \frac{1}{17} \| \Delta\|_2^2 \right) \;,
    \end{talign*}
    where the last line follows by choosing $\delta = 8\| \Delta \|_2 / 17$. Noting the RHS of the last inequality is $o\left( - \frac{1}{32}\|\Delta\|_2^2 \right)$ completes the proof.
\end{proof}

\section{Proof of Proposition~\ref{prop: spKSD validity}}
\label{proof: spksd validity}
\begin{proof}
    The stated assumptions ensure that, for all $\calK \in \calS$, $\ksd(Q, P; \calK)$ is well defined, and that $\ksd(Q, P; \calK_\textrm{id}) = \ksd(Q, P) = 0 \iff Q = P$ (see, e.g., \citet[Theorem 2.2]{chwialkowski2016kernel}). The desired result then follows since $\calK$ is $P$-invariant for all $\calK$ and $\ksd(Q, P; \calS) \geq \ksd(Q, P)$.
\end{proof}

\section{Proof of Proposition~\ref{prop: asymptotic distribution of spKSD}}
\label{proof: asymptotic distribution of spKSD}
By \citet[Sections 5.5.1, 5.5.2]{serfling2009approximation}, sufficient conditions for the stated results are 
\begin{enumerate}
    \item[C1.] $\bbE_{w, w' \sim R_Q}[\tilde{u}_P(w, w')^2] < \infty$.
    \item[C2.] Under $H_0$: $\xi_1 \coloneqq \textrm{Var}_{w \sim R_Q}(\bbE_{w' \sim R_Q}[\tilde{u}_P(w, w')]) = 0$ and $\xi_2 \coloneqq \textrm{Var}_{w, w' \sim R_Q}(\tilde{u}_P(w, w')) > 0$.
    \item[C3.] Under $H_1$: $\xi_1 > 0$.
\end{enumerate}
Now C1.\ holds by assumption. To show C2., we start with the decomposition where, for fixed $w = (x^1, \ldots, x^S) \in \calX^S$,
\begin{talign}
    \bbE_{w' \sim R_Q}[\tilde{u}_P(w, w')]
    = \sum_{s=1}^S\bbE_{(y^1, \ldots, y^S) \sim R_Q}[u_P(x^s, y^s)] 
    = \sum_{s=1}^S\bbE_{y^s \sim \calK_s Q}[u_P(x^s, y^s)] \;,  \label{eq: spKSD asymptotic distribution}
\end{talign}
where the first equality holds as for $w' \sim R_Q$ we can write $w' = (y^1, \ldots, y^S)$ for some $y^s \in \calX$ by construction, and the second equality follows since the marginal distribution of $y^s$ is $\calK_s Q$. When $Q = P$, each term in (\ref{eq: spKSD asymptotic distribution}) equals to $\bbE_{y^s \sim P}[u_P(x^s, y^s)]$ by $P$-invariance. Now, under the assumed conditions in Prop.~\ref{prop: spKSD validity}, the same argument in the proof of \citet[Theorem 4.1]{liu2016kernelized} shows that $\bbE_{y^s \sim P}[u_P(x^s, y^s)] = 0$. Hence, $\xi_1 = 0$.

To prove $\xi_2 > 0$ when $Q = P$, we suppose for a contradiction that $\xi_2 = 0$. We then must have $\tilde{u}_P \equiv c'$ $P$-almost surely for some fixed constant $c'$. Since $P$ admits a positive density on $\calX$ by assumption, this implies that $ c' = 0$. On the other hand, for any probability measure $Q'$ on $\calX$, taking expectation with respect to $R_{Q'}$ yields
\begin{talign*}
    0 = c' 
    &= \bbE_{w, w' \sim R_{Q'}}[\tilde{u}_P(w, w')] \\
    &= \sum_{s = 1}^S \bbE_{(x^1, \ldots, x^S), (y^1, \ldots, y^m) \sim R_{Q'}}[ u_P(x^s, y^s) ] \\
    &= \sum_{s = 1}^S \bbE_{x^s, y^s \sim \calK_s Q'}[ u_P(x^s, y^s) ] \\
    &= \sum_{s = 1}^S \ksd(\calK_s Q', P) \\
    &\geq \ksd(Q', P) \;,
\end{talign*}
where the second identity follows from a similar argument in (\ref{eq: spKSD asymptotic distribution}), and the last line holds because $\calK_\textrm{id} \in \calS$. This is a contradiction, as it would imply $\ksd(Q', P) = 0$ for any $Q' \neq P$.

To show C3, we prove the contrapositive by supposing $\xi_1 = 0$ and aiming to show that $Q = P$. If $\xi_1 = 0$ then there must exist a constant $c$ for which 
\begin{talign*}
    c = \bbE_{w' \sim R_Q}[\tilde{u}_P(w, w')] \;, 
\end{talign*}    
for all $w$ $R_Q$-almost surely. With the stated choice of $\calK$ and the assumption that $Q$ admits a Lebesgue density, the above identity also holds for all $w$ $R_P$-almost surely, where $R_P$ is constructed in the same way as $R_Q$ by replacing $Q$ with $P$. Taking expectation of both sides with respect to $w \sim R_P$ then yields 
\begin{talign*}
    c 
    = \bbE_{w \sim R_P} \bbE_{w' \sim R_Q} [\tilde{u}_P(w, w')]
    = \bbE_{w' \sim R_Q} \bbE_{w \sim R_P} [ \tilde{u}_P(w, w') ] \;,
\end{talign*}
where the last equality follows from the Fubini-Toneli Theorem. Following the same argument above for $Q = P$, we conclude that $\bbE_{w \sim R_P} [ \tilde{u}_P(w, w') ] = 0$, and hence 
\begin{talign*}
    0
    = \bbE_{w' \sim R_Q} \bbE_{w \sim R_Q} [ \tilde{u}_P(w, w') ]
    = \sum_{s=1}^S \ksd(\calK_s Q, P)
    \geq \ksd(Q, P) \;.
\end{talign*}
It follows that $\ksd(Q, P) = 0$, thus $Q = P$.




\section{Validity of the Accept-Reject Rule}
\label{appendix: balancing equation}

\subsection{A Sufficient Condition}
Let $\calK$ be the Markov transition kernel studied in Sec.~\ref{subsec: choosing the proposal density}. We first present a sufficient condition for the detailed balance equation (\ref{eq: detailed balance}):
\begin{talign}
    \int_{x \in A} \sum_{u \in \calU} \delta_{x'}(B) p(x) g(u) \alpha(x, x') dx
    =  \int_{x'\in B} \sum_{u' \in \calU} \delta_{x}(A) p(x') g(u') \alpha(x', x) dx' ,
    \label{eq: detailed balance restated}
\end{talign}
for all $A, B \in \mathcal{B}(\calX)$. For simplicity, we have written $x' = h(x | u)$ and $x = h^{-1}(x' | u')$, so that the dependence of $x'$ on $u$ and of $x$ on $x'$ is implicit. 

\begin{proposition}
\label{prop: detailed balance sufficient condition}
    Let $p$ be a probability density function on $\calX \subset \bbR^d$. Suppose that $h$ is a deterministic, invertible function that is differentiable with differentiable inverse. Furthermore, let $g$ be a known density defined on some discrete space $\calU$. Consider a Markov transition kernel of the form
    \begin{talign}
        \calK(x, A) = \sum_{u \in \calU} \delta_{x'}(A) g( u) \alpha(x, x') + \delta_x(A) r(x) ,
        \label{eq: transition kernel restated}
    \end{talign}
    where $x' \coloneqq h(x | u)$, $\delta_x(A) = 1$ if $x \in A$ and $0$ otherwise, and $r(x) = 1 - \sum_{ u \in \calU} g(u(x, x')) \alpha(x, x')$. Then an accept-reject rule $\alpha(x, x')$ satisfies the detailed balance condition (\ref{eq: detailed balance restated}) if 
    \begin{talign}
        p(x) g(u) \alpha(x, x')
        = p(x') g(u') \alpha(x', x) \left| \frac{\partial h(x | u)}{\partial x} \right| \;.
        \label{eq: detailed balance sufficient condition}
    \end{talign}
\end{proposition}

\begin{proof}

    The proof largely imitates \citet[Sec.~2.1]{green2009reversible}, which shows the claim when the density $g$ is defined on a continuous space. Defining $\calU_B \coloneqq \{ u: x' = h(x | u) \in B \textrm{ for some } x \in \calX\}$ and $\calU_A \coloneqq \{ u: x = h^{-1}(x' | u) \in B \textrm{ for some } x' \in \calX\}$, we can rewrite (\ref{eq: detailed balance restated}) as
    \begin{talign*}
        \int_{x \in A} \sum_{u \in \calU_B} p(x) g(u) \alpha(x, x') dx
        =  \int_{x'\in B} \sum_{u' \in \calU_A} p(x') g(u') \alpha(x', x) dx' \;.
    \end{talign*}
    Noting that $(x, u) \in A \times \calU_B \iff (x', u') = (h(x | u), u) \in B \times \calU_A $, and by the invertibility of the transformation $h$, a change-of-variable formula can be applied to the right-hand-side of (\ref{eq: detailed balance restated}) to yield
    \begin{talign*}
        \int_{x \in A} \sum_{u \in \calU_A} p(x) g(u) \alpha(x, x') dx
        =  \int_{x \in A} \sum_{u \in \calU_A} p(x') g(u) \alpha(x', x) \left| \frac{\partial h(x | u)}{\partial x} \right| dx .
    \end{talign*}
    We therefore conclude that a sufficient condition is
    \begin{talign*}
        p(x) g(u) \alpha(x, x')
        = p(x') g(u') \alpha(x', x) \left| \frac{\partial h(x | u)}{\partial x} \right| \;.
    \end{talign*}
\end{proof}

In particular, it follows that the detailed balance condition holds with 
\begin{talign}
    \alpha(x, x') = \min\left(1, \frac{p(x')g'(u')} { p(x) g(u)} \left| \frac{\partial h(x | u)}{\partial x} \right| \right) ,
    \label{eq: mh rule restated}
\end{talign}
by verifying that it indeed satisfies (\ref{eq: detailed balance sufficient condition}). This can be viewed as a generalisation of the Metropolis-Hastings (MH) rule $\alpha(x, x') = \min\left(1, \frac{p(x')g'(u')} { p(x) g(u)} \right) $.

\subsection{A Class of Valid Accept-Reject Rules}
Accept-reject rules of the form (\ref{eq: mh rule restated}) is not the only choice that satisfies the detailed balance condition. For the standard Metropolis-Hastings transition kernel, alternative accept-reject rules have been studied \citep{barker1965, peskun1973optimum, hird2020fresh}. We follow \citet{hird2020fresh} to propose a class of accept-reject rules that are valid for proposed kernels of the form (\ref{eq: transition kernel restated}).

\begin{lemma}
    Using the same notations in Prop.~\ref{prop: detailed balance sufficient condition}, define $t(x, x') \coloneqq \frac{p(x') g(u')}{p(x) g(u)}$ when $p(x) g(u) > 0$, and $t(x, x') = 0$ otherwise, where $u, u' \in \calU$ such that $x' = h(x | u)$ and $x = h(x' | u')$. Then the equality (\ref{eq: detailed balance sufficient condition}) holds for 
    \begin{talign*}
        \alpha(x, x') = \rho\left( \left| \frac{\partial h(x | u)}{\partial x} \right| t(x, x') \right), 
    \end{talign*}
    where $\rho$ is any function that satisfies $\rho(s) = s \rho(1 / s)$, for all $s > 0$, and $\rho(0) \coloneqq 0$.
\end{lemma}

\begin{proof}
    We follow the derivation in \citet[Eq.~4]{hird2020fresh}. By the definition of $t$, it is obvious that $t(x, x') = 1 / t(x', x)$ and $p(x) g(u) t(x, x') = p(x') g(u')$. The assumption on $\rho$ then implies
    \begin{talign}
        p(x) g(u) \alpha(x, x')
        &= p(x) g(u) \rho\left( \left| \frac{\partial h(x | u)}{\partial x} \right| t(x, x') \right) \nonumber \\
        &= p(x) g(u) t(x, x') \left| \frac{\partial h(x | u)}{\partial x} \right| \rho\left( \frac{1}{\left| \frac{\partial h(x | u)}{\partial x} \right| t(x, x')} \right) \nonumber \\
        &= p(x') g(u') \left| \frac{\partial h(x | u)}{\partial x} \right| \rho \left( \left| \frac{\partial h(x' | u')}{\partial x } \right| t(x', x) \right)  \label{eq: lemma balancing eqn intermediate step} \\
        &= p(x') \alpha(x', x) \left| \frac{\partial h(x | u)}{\partial x} \right| , \nonumber 
    \end{talign}
    where (\ref{eq: lemma balancing eqn intermediate step}) follows from the fact that $\left| \frac{\partial h(x' | u')}{\partial x' } \right| = \left| \frac{\partial h(x | u)}{\partial x} \right|^{-1}$ by the invertibility of $h$.
\end{proof}

In particular, choosing $\rho(t) = \min(1, t)$ gives the generalised MH accept-reject rule (\ref{eq: mh rule restated}). Another feasible choice is $g(t) = t / (1 + t)$, which leads to a generalised version of the Barker's rule \citep{barker1965, peskun1973optimum, livingstone2021barker}
\begin{align*}
    \alpha(x', x)
    =  
    \frac{p(x')g(u') \left| \frac{\partial h(x | u)}{\partial x} \right|}{p(x) g(u) + p(x')g(u') \left| \frac{\partial h(x | u)}{\partial x} \right|} ,
\end{align*}
whenever $p(x) g(u) > 0$, and 0 otherwise.

\section{Proof of Proposition~\ref{prop: limiting distribution}}
\label{pf: limiting distribution}
We first characterise the limiting distribution when the initial distribution is a point mass $\delta_{x_0}$ for any $x_0 \in \bbR^d$, then generalise the result to an arbitrary probability measure $Q$.

\subsection{Limiting Distribution with a Point Mass Initial Distribution}
Fixing $x_0 \in \bbR^d$, we define $\calI_{x_0} \coloneqq \{ x_0 + k\nu: k \in \bbZ \}$. We first identify a stationary distribution and aim to show that it is also the limiting distribution.

\begin{lemma}
\label{lem: stationary distribution}
    The following probability mass function defines a stationary distribution under $\calK$:
    \begin{talign*}
        r_{x_0}(x)
        = \frac{p(x)}{\sum_{k \in \bbZ} p(x_0 + k\nu)} \ ,
    \end{talign*}
    if $x \in \calI_{x_0}$, and $r_{x_0}(x) = 0$, otherwise. 
\end{lemma}

\begin{proof}
    A sufficient condition for the detailed-balance condition in this case is
    \begin{talign*}
        r_{x_0}(x) g(u) \alpha(x, x')
        =  r_{x_0}(x') g'(u') \alpha(x', x) ,
    \end{talign*}
    for all $x, x' \in \bbR^d$. Fix $x$. Since $\alpha(x, x') = 0$ unless $x' \in \{ x - \nu, x + \nu \}$, it is sufficient to check whether the above equation holds for $x' \in \{ x - \nu, x + \nu \}$.  For, e.g., $x' = x + \nu$,
    \begin{talign*}
        \textrm{LHS}
        &= \frac{p(x)}{\sum_{k \in \bbZ} p(x_0 + k\nu)} g(u) \min\left(1, \frac{g(u') p(x')}{g(u) p(x)} \right) \\
        &= \frac{1}{\sum_{k \in \bbZ} p(x_0 + k\nu)} g(u') \min\left( p(x), p(x') \right) , \qquad &\textrm{as $g(u) = g(u')$ by definition.} \\
        &= \frac{p(x')}{\sum_{k \in \bbZ} p(x_0 + k\nu)} g(u') \min\left( \frac{g(u) p(x)}{g(u') p(x')}, 1 \right) , \qquad &\textrm{again by $g(u) = g(u')$.} \\
        &= \textrm{RHS}.
    \end{talign*}
    A similar derivation for $x' = x - \nu$ completes the proof.
\end{proof}


The next result shows that the Markov chain defined on $\calI_{x_0}$ is irreducible and aperiodic under mild conditions.

\begin{lemma}
\label{lem: irreducible and aperiodic}
    Given $x_0 \in \calX$ and consider the Markov chain with initial distribution $\delta_{x_0}$. Then 
    \begin{enumerate}
        \item All state in $\calI_{x_0}$ are irreducible.
        \item A state $x \in \calI_{x_0}$ is aperiodic if $p(x + \nu) < p(x)$ or $p(x - \nu) < p(x)$. 
    \end{enumerate}
\end{lemma}

\begin{proof}
    To prove \emph{1.}, it is sufficient to show that any state $x \in \calI_{x_0}$ can reach any other state $y \in \calI_{x_0}$ with positive probability, i.e., for any singleton $A = \{ y \}$ where $y \in \calI_{x_0}$, there exists $T \in \mathbb{N}$ so that $\calK^T(x, A) > 0$, where 
    \begin{talign*}
        \calK(x, A)
        = \sum_{u \in \calU} \delta_{x'}(A) g( u) \alpha(x, x') + \delta_x(A) r(x) ,
    \end{talign*}
    where $\calU = \{ (1, 2), (2, 1) \}$, $g(u) = 1/2$ for all $u \in \calU$, and $x' = x'(x, u) = x + \theta(\mu_{u_1} - \mu_{u_2})$ (see Section~\ref{sec: power intuition}). We fix $x \in \calI_{d_0}$ and pick $y \in \{ x - \nu, x + \nu \}$, i.e., $y$ is the point immediately to the left or right of $x$ in $\calI_{x_0}$. The transition probability in this case reduces to
    \begin{talign*}
        \calK(x, \{ y \})
        = \frac{1}{2} \alpha(x, y)
        = \frac{1}{2} \min\left(1, \frac{p(y)}{p(x)} \right),
    \end{talign*}
    which is positive as $p$ is positive on $\bbR^d$, i.e., $x$ can move to its left or right state in one step with positive probability. An inductive argument directly shows that $x$ can move to \emph{any} $y \in \calI_{x_0}$ with positive probability in finitely many steps. This shows \emph{1.}
    
    To prove \emph{2.}, we note that if $p(x + \nu) < p(x)$ or $p(x - \nu) < p(x)$, then the 1-step transition probability of starting from $x$ and staying is non-zero. Indeed,
    \begin{talign*}
        \calK(x, \{x\}) 
        = 1 - \calK(x, \{x + \nu\} ) - \calK(x, \{x - \nu\}) 
        = 1 - \frac{1}{2} \alpha(x, x + \nu) - \frac{1}{2} \alpha(x, x - \nu) .
    \end{talign*}
    Since $p(x + \nu) < p(x)$, we have $\alpha(x, x + \nu) = \min\left(1, p(x + \nu) / p(x) \right) < 1$. Similarly, $\alpha(x, x - \nu) < 1$. Hence, $\calK(x,  \{ x \} ) > 0$, thus $x$ is aperiodic. 
\end{proof}

Combining Lemma~\ref{lem: stationary distribution} and \ref{lem: irreducible and aperiodic}, we can identify the limiting distribution when the initial distribution is a point mass at $x \in \calI_{x_0}$. 
\begin{proposition}
\label{corollary: limiting distribution point mass}
    If $p(x + \nu) < p(x)$ or $p(x - \nu) < p(x)$ for all $x \in \calI_0$, then $r_{x_0}$ is also the unique limiting distribution, i.e., $\calK^T(x, A) \to \sum_{x' \in A} r_{x_0}(x')$, for all $x \in \calI_{x_0}$ and $A \subset \calI_{x_0}$. Furthermore, for a Lebesgue-measurable set $A \subset \calX$ and a state $x \in \calX$,
    \begin{talign}
        \lim_{n \to \infty} \calK^n(x, A)
        = \sum_{x' \in \calI_x} \delta_{x'}(A) \frac{p(x')}{\sum_{k \in \bbZ} p(x + k \nu)} \ .
        \label{eq: limiting distribution point mass}
    \end{talign}
\end{proposition}
\begin{proof}
    Under the stated assumption, Lemma~\ref{lem: irreducible and aperiodic} shows that the Markov chain is irreducible and aperiodic on $\calI_{x_0}$. Since the stationary distribution of an irreducible and aperiodic Markov chain defined on a countable space is also the unique limiting distribution (e.g., \citet{meyn2012markov}), the first part follows. (\ref{eq: limiting distribution point mass}) holds because, for any $x$, $r_x$ is a probability mass function taking zero values outside of $\calI_x$.
\end{proof}

\subsection{Limiting Distribution with a General Initial Distribution}
We now prove Proposition~\ref{prop: limiting distribution}, which characterises the limiting distribution with a general initial distribution $Q$.

\begin{proof}[Proof of Proposition~\ref{prop: limiting distribution}]
    Let $A \subset \calX$ be Lebesgue-measurable. For any fixed $n \in \bbZ_+$, the probability of $A$ under the $n$-step perturbed distribution is
    \begin{talign*}
        (\calK^n Q)(A)
        = \int_{x \in \calX} \calK^n(x, A) Q(dx)
        = \int_{x \in \calX} \calK^n(x, A) q(x) dx .
    \end{talign*}
    Since $| \calK^n(x, A) | \leq | \calK^n(x, \calI_x) | \leq 1$ and $\int_{x \in \calX} q(x) dx = 1 < \infty$, we can apply Dominated Convergence Theorem \citep[Section~5.6]{bartle2014elements} to conclude
    \begin{talign}
        \lim_{n \to \infty} (\calK^n Q)(A)
        &= \int_{x \in \calX} \calK^\infty(x, A) q(x) dx \nonumber \\
        &=\int_{x \in \calX} \sum_{x' \in \calI_x} \delta_{x'}(A) \frac{p(x')}{\sum_{k \in \bbZ} p(x + k\nu)} q(x) dx  , &\textrm{by Corollary~\ref{corollary: limiting distribution point mass}.} \nonumber \\
        &=\int_{x \in \calX} \sum_{s \in \bbZ} \delta_{x + s\nu}(A) \frac{p(x + s\nu)}{\sum_{k \in \bbZ} p(x + k\nu)} q(x) dx \nonumber \\
        &= \sum_{s \in \bbZ} \int_{x \in \calX} \delta_{x + s\nu}(A) \frac{p(x + s\nu)}{\sum_{k \in \bbZ} p(x + k\nu)} q(x) dx \label{eq: fubini 1} \\
        &= \sum_{s \in \bbZ} \int_{u \in \calX} \delta_u(A) \frac{p(u)}{\sum_{k \in \bbZ} p(w + (k-s)\nu)} q(u - s\nu) du \label{eq: change of var}\\
        &= \int_{u \in \calX} \sum_{s \in \bbZ} \delta_u(A) \frac{p(u)}{\sum_{k \in \bbZ} p(w + (k-s)\nu)} q(u - s\nu) du \label{eq: fubini 2} \\
        &= \int_{u \in \calX} \delta_u(A) \frac{p(u)}{\sum_{k \in \bbZ} p(w + (k-s)\nu)} \sum_{s \in \bbZ} q(u - s\nu) du \nonumber \\
        &= \int_{u \in A} p(u) \frac{\sum_{s \in \bbZ} q(u + s\nu)}{\sum_{k \in \bbZ} p(u + k\nu)} du , \nonumber 
    \end{talign}
    where in (\ref{eq: fubini 1}) we have applied Fubini-Toneli Theorem \citep[Section~10.9, 10.10]{bartle2014elements}, (\ref{eq: change of var}) follows from a change of variable $u \coloneqq x + s\nu$ for a given $s$, (\ref{eq: fubini 2}) follows from Fubini-Toneli Theorem again, and the last line holds by a re-indexing of the sums on the numerator and on the denominator. In particular, we can apply Fubini-Toneli Theorem as $\int_{x\in\calX}  \sum_{s \in \bbZ} |\delta_{x + s\nu}(A) \frac{p(x + s\nu)}{\sum_{k \in \bbZ} p(x + k\nu)} | q(x) dx \leq \int_{x\in\calX}  \sum_{s \in \bbZ} \frac{p(x + s\nu)}{\sum_{k \in \bbZ} p(x + k\nu)} q(x) dx  < \infty$.
\end{proof}

    
\begin{algorithm}[!t]
\caption{Estimating mode vectors and Hessians \citep[Algorithm~3]{pompe2020framework}}
\label{alg: find mode approximations}
    \begin{algorithmic}
        \STATE {\bfseries Input:} Initial points $s_1, \ldots, s_{M_0}$, small positive value $\beta$.
        \STATE {\bfseries Output:} Approximates for mode vectors $\{ \mu_1, \ldots, \mu_M\}$.
        \STATE Initialise BFGS at points $s_1, \ldots, s_{M_0}$ and run the algorithm to minimise $- \log p(x)$. 
        \STATE Denote the returned estimates of the local optima by $m_1, \ldots, m_{M_0}$ and their corresponding Hessian matrices by $A_1, \ldots, A_{M_0}$.
        \STATE Set $\mu_1 \coloneqq m_1$, $A_{\mu_1} \coloneqq A_1, M = 1$.
        \FOR{$i = 2, \ldots, M_0$}
            \IF{$\min_{j \in \{1, \ldots, M \}} \frac{1}{2} ((\mu_j - m_i)^\top A_{\mu_j} (\mu_j - m_i) + (\mu_j - m_i)^\top A_i (\mu_j - m_i)) < \beta$}
                \STATE $k \coloneqq \arg\min_{j \in \{1, \ldots, M \}} \frac{1}{2} ((\mu_j - m_i)^\top A_{\mu_j} (\mu_j - m_i) + (\mu_j - m_i)^\top A_i (\mu_j - m_i))$.
                \IF{$p^\ast(\mu_k) < p^\ast(m_i)$}
                    \STATE Set $\mu_k \coloneqq m_i$ and $A_{\mu_k} \coloneqq A_i$.
                \ENDIF
            \ELSE
                \STATE $\mu_{M + 1} \coloneqq m_i$ and $A_{\mu_{M + 1}} \coloneqq A_i$.
                \STATE $M \coloneqq M + 1$.
            \ENDIF
        \ENDFOR
    \end{algorithmic}
\end{algorithm}

\section{Implementation Details}
This section holds details about the practical implementation of the spKSD and the ospKSD methods.

\subsection{Finding Mode Vectors via Optimisation and Merging}
\label{subsec: finding the mode vectors via optimisation}

\paragraph{Finding local modes}
In practice, the mode locations and Hessians of the density of a non-trivial target distribution are rarely available. \citet{pompe2020framework} describes a general framework to estimate these quantities. It proceeds by running in parallel a sequence of optimisers initiated at different starting points. This is done by minimising the objective $-\log p$ using the BFGS algorithm \citep{nocedal2009numerical}, which returns both the local minima and the approximated Hessian at those points. In our experiments, we use the BFGS algorithm and run for at most 1000 iterations with each initial point.

\paragraph{Mode merging}
Although the end points of the optimisation procedure starting from different initial points may lie close to the local minima, they will still be numerically different from each other. \citet{pompe2020framework} proposed to merge two end points $m_i$ and $m_j$ if their Mahalanobis distance weighted by the averaged Hessians at those points is below a given threshold $\beta$. The full procedure is stated in Algorithm~\ref{alg: find mode approximations} for completeness.

\paragraph{Choosing the initial points}
A set of $n_\textrm{init}$ initial points for BFGS can be constructed either by sampling randomly from a product of intervals $[L_1, U_1] \times \cdots \times [L_d, U_d]$ in $\calX$, or simply from a held-out training set. The first approach will allow modes not covered by the training data to be detected, and the second approach can lead to faster convergence of the optimisation algorithm when the training points lie near the modes of $P$. For spKSD, the first approach is used as we do not assume a held-out set is available. For ospKSD, to combine the best of the two approaches whilst maintaining the same computational budget, half of the $n_{\textrm{init}}$ initial points are drawn randomly from the training set and the other half are initialised uniformly from $[L_1, U_1] \times \cdots \times [L_d, U_d]$. 

\subsection{Choosing the Number of Transitions $T$}
The number of transitions $T$ dictates the perturbed distribution, thus impacting the performance of spKSD. Intuitively, $T$ should be set to a large value when the acceptance rate is low to ensure the limiting distribution is achieved. The spKSD could suffer from a low acceptance rate when the estimates of the modes and local Hessians of the target distribution are inaccurate, or when the target distribution cannot be approximated by a mixture of elliptic distributions.

We propose two heuristics to choose this hyper-parameter in practice: \emph{(i)} viewing this as another hyper-parameter and tuning it using a training set by selecting from a pre-specified set of values, or \emph{(ii)} setting it to a large value (e.g., $T = 1000$) if the computational budget allows. 

In particular, we recommend a large $T$ because, with the proposed transition kernel $\calK$, the KSD $\ksd(\calK^T Q, \calK^T P)$ between the perturbed distributions does \emph{not} necessarily decrease as $T$ grows. This is because $\calK^T Q$ does \emph{not} necessarily converge to $P$ as $T \to \infty$, since $\calK$ is \emph{not} irreducible (see the discussions in Section~\ref{subsec: choosing the proposal density}). 

\begin{figure}
    \centering
    \includegraphics[width=0.32\textwidth]{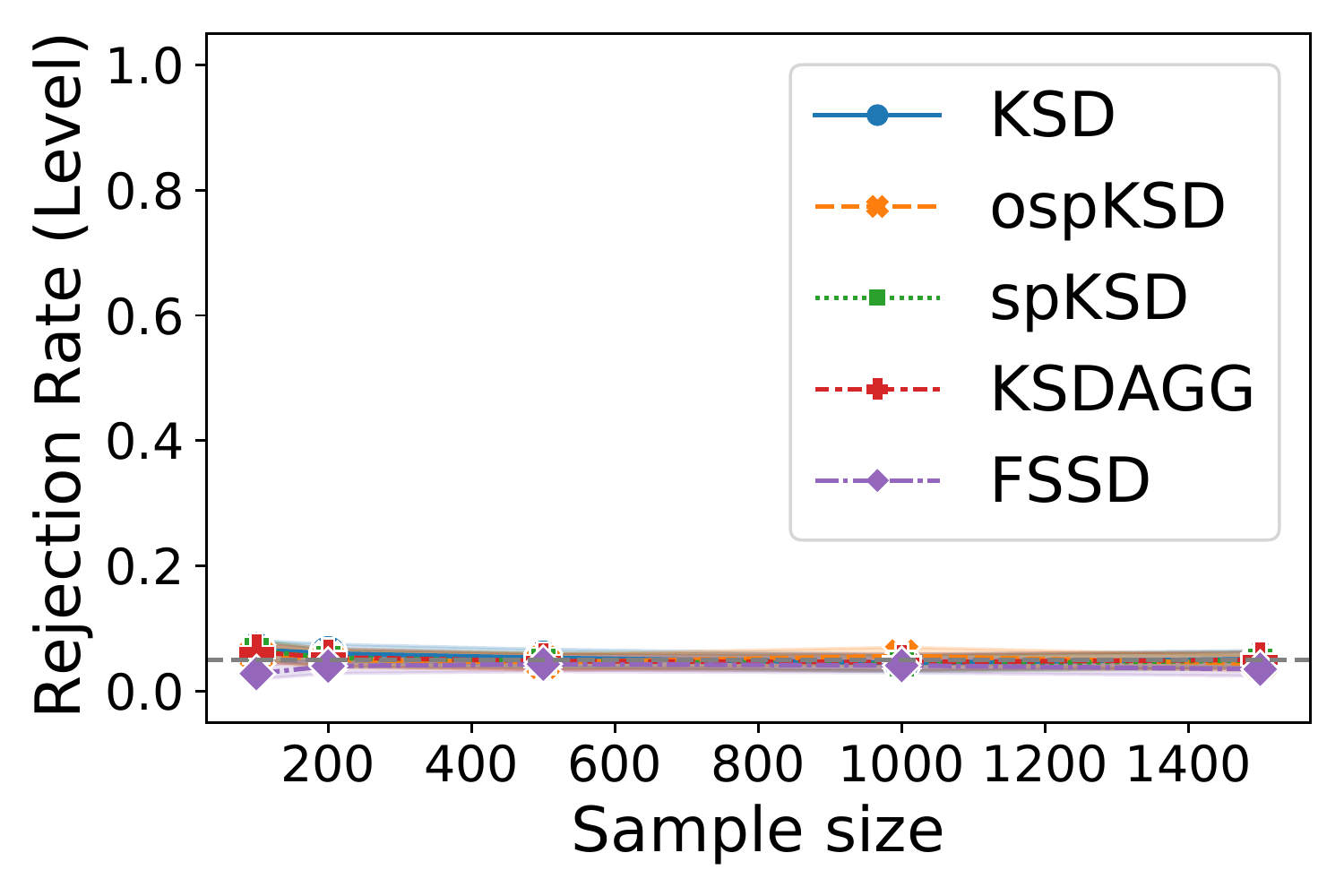}
    \includegraphics[width=0.32\textwidth]{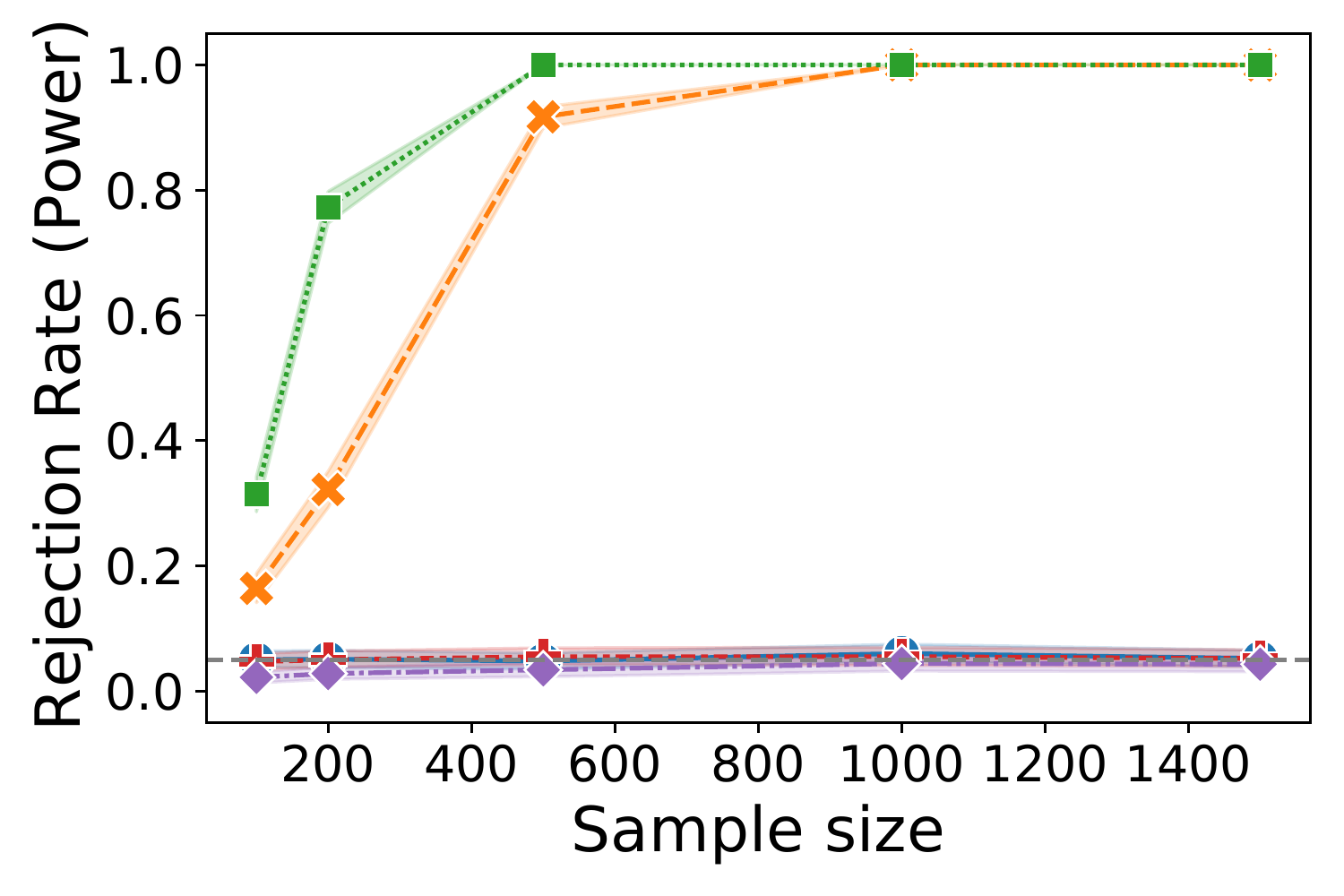}
    \includegraphics[width=0.8\textwidth]{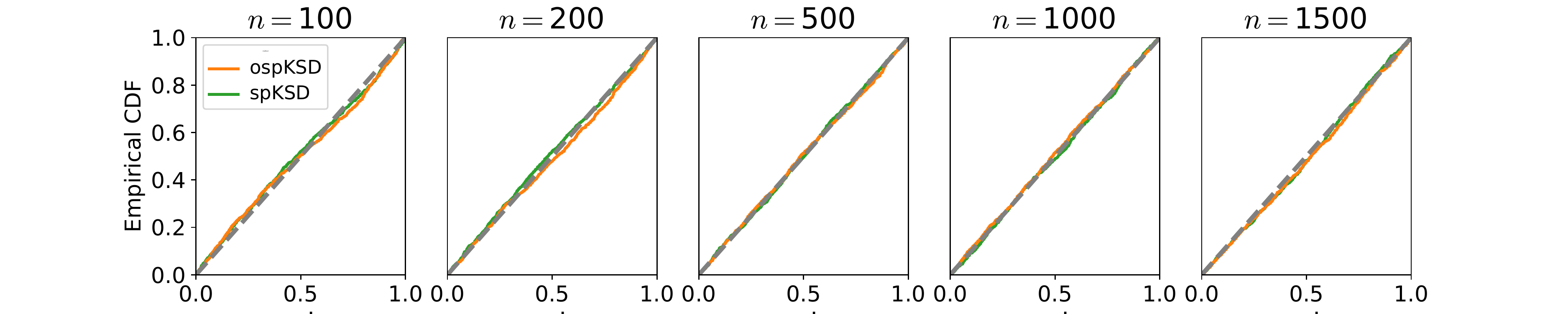}
    \caption{Level (top left) and power (top right) experiments with the multivariate Gaussian mixture example. Bottom: Empirical cumulative distribution function (CDF) of the $p$-values in the level experiments, where the grey dashed line is the CDF of a uniform distribution on $[0, 1]$.}
    \label{fg: bimodal level and power}
\end{figure}

\section{Experimental Details and Supplementary Plots}
\label{appendix: details on experiments}
In this section, we provide detailed definition of the distributions used in each experiment, as well as supplementary figures, including a level study and a power study of the proposed methods.

\subsection{Multivariate Gaussian Mixture: Supplementary Plots}
We include a level study and a power study for the spKSD and ospKSD tests. The target distribution is a multivariate Gaussian mixture in $50$ dimensions, with density $p(x) \propto \pi_p \exp\left(- \frac{1}{2} \| x \|_2^2 \right) + (1 - \pi_p) \exp\left(- \frac{1}{2} \|x - \Delta e_1\|_2^2 \right)$, where $\pi_p = 0.5$, $\Delta = 6$, and $e_1\in \bbR^d$ is a vector with 1 in the first coordinate and 0 in others. Samples are drawn either from the same distribution (level experiment), or from only the left component (power experiment). The probability of rejection over 100 repetitions is plotted in Fig.~\ref{fg: bimodal level and power}. We can see from the top left plot that under the null hypothesis, all tests have the prescribed test level $\alpha = 0.05$. The plots in the bottom row further confirm the validity of the test by showing that the empirical cumulative distribution function (CDF) of the $p$-values under the null is indeed close to the CDF of a uniform distribution. In the top right plot, we investigate the rejection rate under the alternative hypothesis where the samples are drawn from the left mode only. We see that both ospKSD and spKSD achieve a significantly higher power than the benchmarks (KSD, \textsc{KSDAgg} and FSSD), whose power remains close to the level for all sample sizes.

\subsection{Mixture of $t$ and Banana Distributions}
\label{appendix: t banana}
The mixture of $t$ and banana example mostly follows the setup in \citet{pompe2020framework}. Each $t$-distribution has 7 degrees of freedom and covariance matrix $0.1 \sqrt{d} I_d$. Each banana-shaped distribution has density function $p_{b, \mu} = p_{\mu} \circ \phi_b$, where $\phi_{b, \mu}(x) = (x_1, x_2 + b x_1^2 - 100b, x_3, \ldots, x_d)^\top$ with $b=0.003$, and $p_{\mu}$ is the density of a $t$-distribution with 7 degrees of freedom centred at $\mu \in \bbR^d$ and with shape matrix $C = \textrm{diag}(100, 1, \ldots, 1) \in \bbR^{d \times d}$.

\subsection{Sensors Localisation}
\label{appendix: sensors}
Following \citet{tak2018repelling}, we use a diffuse bivariate Gaussian prior distribution $\calN(0, 10^2 I_2)$ for each $x_i \in \bbR^2$. Let $w_{ij}$ be the binary random variable for which $w_{ij} = 1$ if the distance $y_{ij}$ is observed and 0 otherwise. The full posterior is
\begin{talign*}
	\pi(x_1, \ldots, x_4 | y, w)
	\propto
	\exp\left( - \frac{\sum_{k=1}^4 x_k^\top x_k}{2 \times 10^2} \right) \Pi_{i < j} f_{ij}(x_i, x_j | y_{ij}, w_{ij}) ,
\end{talign*}
where $w = \{ w_{ij} \}$, $y = \{ y_{ij} \}$ and 
\begin{talign*}
    f_{ij}(x_i, x_j | y_{ij}, w_{ij})
    = 
    \left[ \exp\left(- \frac{(y_{ij} - \| x_i - x_j \|_2 )^2}{2 \times 0.02^2} \right)
    \exp\left( \frac{- \| x_i - x_j\|_2^2}{2 \times 0.3^2} \right) \right]^{w_{ij}}
    \left[ 1 - \exp\left( \frac{- \| x_i - x_j\|_2^2}{2 \times 0.3^2} \right)\right]^{1 - w_{ij}} .
\end{talign*}

\subsection{Gaussian-Bernoulli Restricted Boltzmann Machine}
\label{appendix: GB-RBM}
We include a supplementary experiment with Gaussian-Bernoulli Restricted Boltzmann Machines (RBMs) \citep{cho2013gaussian}. This model is a popular benchmark for assessing GOF tests \citep{liu2016kernelized, jitkrittum2017linear, schrab2022ksd}. It is a latent variable model with joint density $p(x, h) \propto \exp\left( \frac{1}{2} x^\top B h + b^\top x + c^\top h - \frac{1}{2} \| x \|_2^2 \right)$, where $h \in \{ -1, 1 \}^{d_h}$, and $B, b, c$ are fixed hyperparameters. The marginal density $p(x)$ can be rewritten as a mixture of Gaussian distributions:
\begin{talign}
    p(x)
    = \sum_{h} \gamma(h) N\left(x; \frac{1}{2}Bh + b, I_d \right), \qquad
    \textrm{where }
    \gamma(h) &\propto \exp\left(\frac{1}{2} \left\| \frac{1}{2} Bh + b \right\|_2^2 + c^\top h \right) .
    \label{eq: GB-RBM density}
\end{talign}
We consider two parameter settings: \emph{standard} and \emph{multi-modal}. The \emph{standard} setting follows the setups in \citet{liu2016kernelized}, in which case the RBM is unimodal. For the target distribution $P$, we randomly sample the entries of $b$ and $c$ from a standard normal, and select the entries of $B$ from $\{-1, 1 \}$ with equal probability. We sample from a perturbed version of $P$ where Gaussian noises with standard deviation $\sigma$ are injected into the entries of $B$. As $\sigma$ increases, the problem becomes easier, so all tests are able to reject with a high probability (Fig.~\ref{fig: rbm standard}).

For the \emph{multi-modal} setting, we set $c = 0$ and $b = 0$, and choose $B$ so that the modes are well-separated. Samples are drawn from the same model with a different $c$, which controls the mixing weights. Specifically, we choose $B = 6 E$, where $E \in \bbR^{d \times d_h}$ is formed by the top $d$ row and $d_h$ columns of the matrix $I_{d_\textrm{max}}$, where $d_\textrm{max} \coloneqq \max(d, d_h)$. This renders the local modes of $p$ to be located at the corners of a hyper-cube of width $6$. Due to the choice of $b$ and $B$, changing $c$ only affects the weights of the components but not their mode locations. We choose $c = 0$ for the target so that all components have equal weights, and draw samples with from the same model with $c = (c_0, c_0, 0, \ldots, 0) \in \bbR^{d_h}$ for some $c_0$ using a Gibbs sampler, which we describe in the next subseciton.

We run ospKSD and spKSD with $T=50$ steps and report the results in Fig.~\ref{fig: rbm}. The left plot shows the rejection probability with different values of $c_0$ and $d_h = 5$. As $c_0$ increases, the weights deviate further from those in the target, which is detected by ospKSD and spKSD as shown by the increasing power. The benchmarks again fail to detect this discrepancy due to the sparsity of the modes. We then fix $c_0 = 5$ and analyse how the performance scales with the latent dimension $d_h$ in the right plot of Fig.~\ref{fig: rbm}. With a moderate $d_h$, the rejection probabilities of ospKSD and spKSD are significantly larger than the others. This gap vanishes for $d_h \geq 20$, since a larger $d_h$ gives rise to more modes in $p$ and hence more possible jump directions for each point. Therefore, the probability of proposing a ``correct'' move at each step declines, leading to a small acceptance rate and thus a low power.

\begin{figure}
   \begin{minipage}[t]{.48\textwidth}
        \centering
        \includegraphics[width=.5\textwidth]{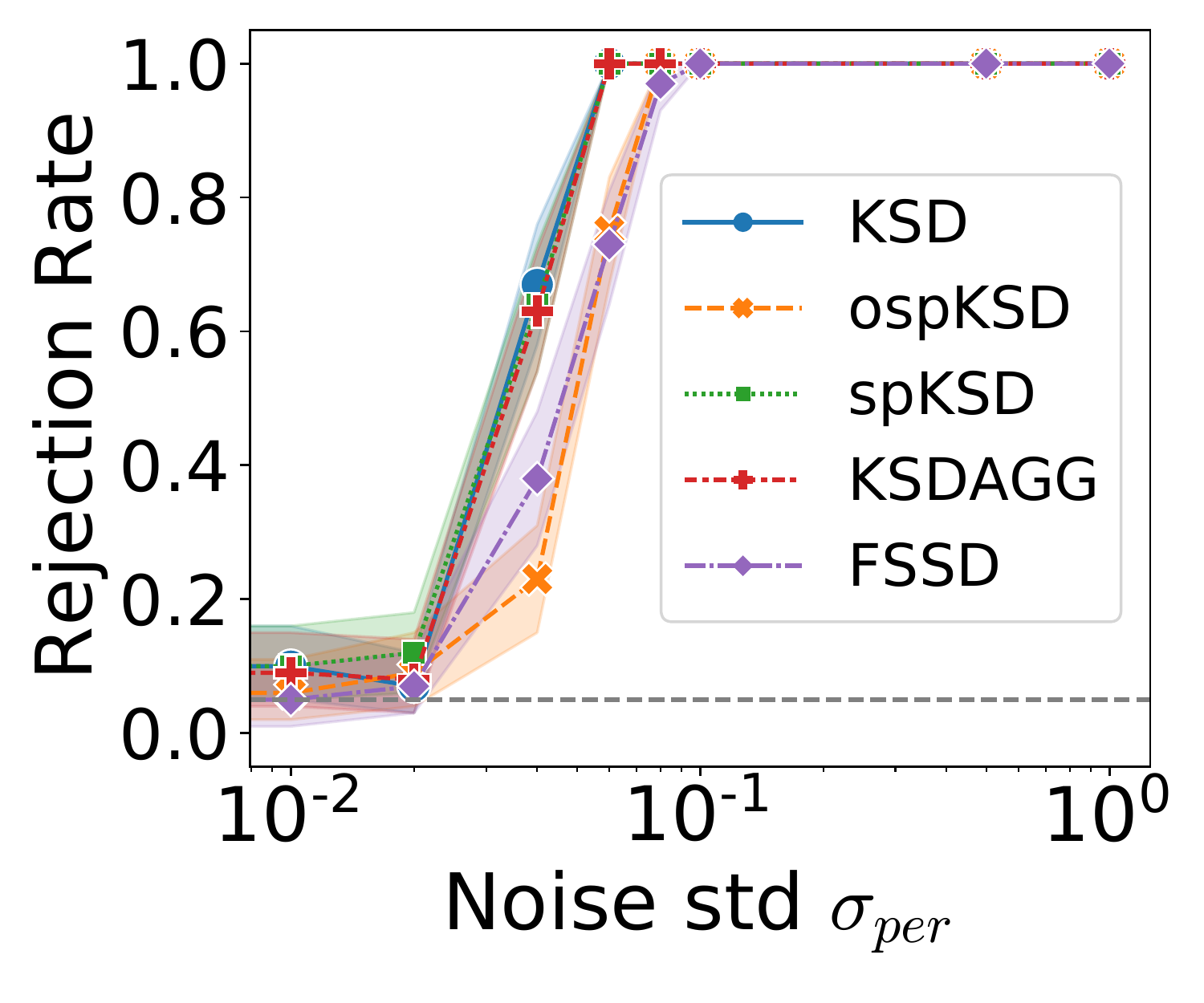}
        \caption{GB-RBM with the standard setting.}
        \label{fig: rbm standard}  
    \end{minipage}
    \hfill
    \begin{minipage}[t]{.5\textwidth}
        \centering
        \includegraphics[width=1.\textwidth]{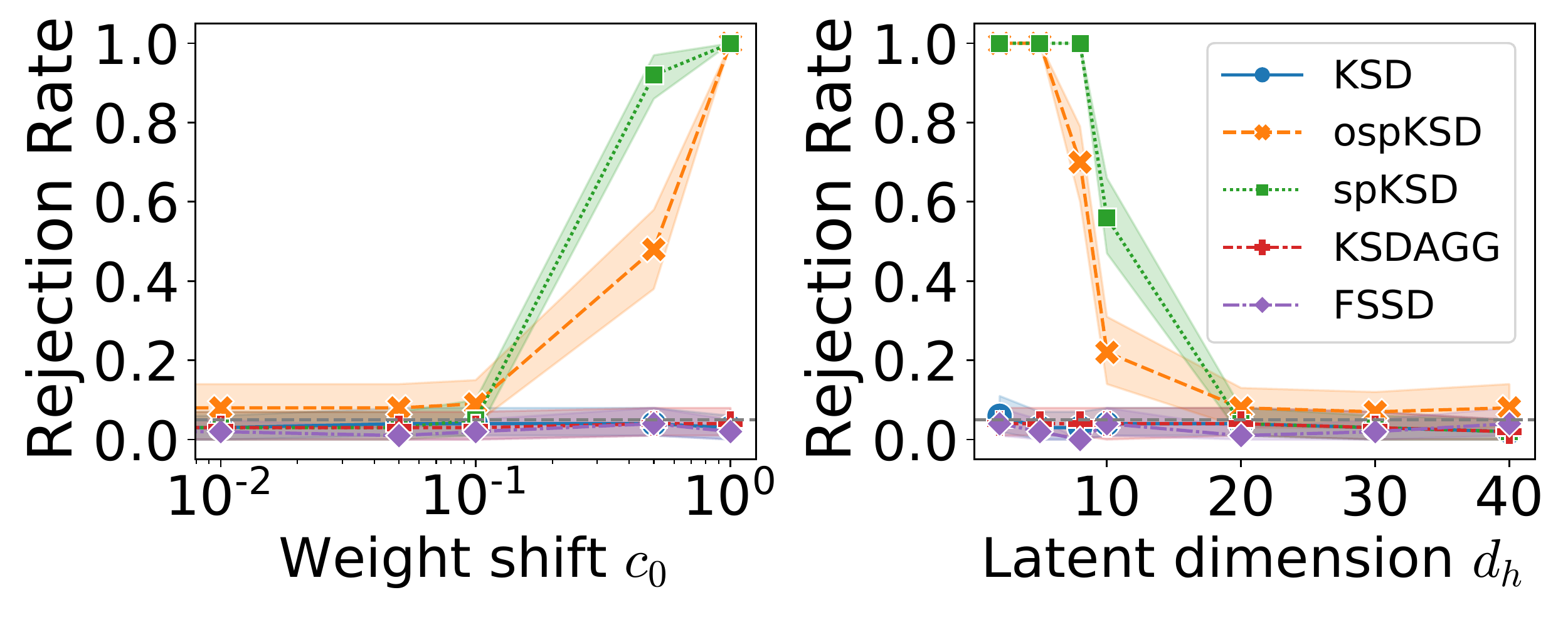}
        \caption{GB-RBM with the multi-modal setting.}
        \label{fig: rbm}
    \end{minipage}
\end{figure}

\subsubsection{Sampling from Gaussian-Bernoulli RBMs with Well-Separated Modes}
\label{appendix: sampling from GB-RBM}
We discuss practical considerations when sampling from the Gaussian-Bernoulli RBM with our particular choice of $B$ (the \emph{multi-modal} setting). Denote by $\textrm{GB-RBM}(B, b, c)$ the joint density of the Gaussian-Bernoulli RBM with parameters $B, b, c$ (see the previous section for definition). 

In the \emph{multi-modal} setting, we set $c = (6, 6, 0, \ldots, 0) \in \bbR^{d_h}$, and $B = \lambda E$, where $\lambda$ is a large, positive constant, and $E \in \bbR^{d \times d_h}$ is the top $d \times d_h$ sub-matrix of $I_{d_\textrm{max}}$, where $d_\textrm{max} \coloneqq \max(d, d_h)$. These lead to a model that is a mixture of Gaussian distributions with \emph{distantly} located modes. More specifically, the modes are located at the corners of the hyper-cube in $\bbR^d$ with length $6$.

The standard way to sample from a GB-RBM is to use a Gibbs sampler (see e.g.,  \citet{melchior2017gaussian, jitkrittum2017linear}). However, the Gibbs sampling can suffered from poor mixing when $\lambda$ is large. This is because the modes will become more disconnected as $\lambda$ increases, thus making it more challenging for the sampler to learn the mixing weights correctly. This means an impractically long burn-in period would be required for the Gibbs sampler to produce faithful samples from the ground truth.

We propose a practical method to generate faithful samples when $\lambda$ is large. It is based on the observation that, with $B$ of the above form, varying the value of $\lambda$ only affects the mean locations of the Gaussian components in the GB-RBM, but \emph{not} the mixing ratios. Indeed, for any $h \in \{ -1, 1 \}^{d_h}$, we have $\| Bh \|_2^2 = \lambda^2 \| Eh \|_2^2 = \lambda^2 d_h$, which is constant for all $h$. Substituting this into $\gamma(h)$ of (\ref{eq: GB-RBM density}), 
\begin{talign*}
    \gamma(h )
    = \frac{\exp\left( \frac{1}{8}\lambda^2 d_h + c^\top h \right)}{ \sum_{h'} \exp\left( \frac{1}{8}\lambda^2 d_h + c^\top h' \right) }
    = \frac{\exp\left(c^\top h \right)}{ \sum_{h'} \exp\left(c^\top h' \right) },
\end{talign*}
which does not depend on $\lambda$. It follows that, given any $\lambda, \lambda' \geq 0$, if $(x, h) \sim \textrm{GB-RBM}(\lambda E, b, c)$, then $(y, h) \sim \textrm{GB-RBM}(\lambda' E, b, c)$, where $y \coloneqq x - \frac{1}{2} (\lambda - \lambda' ) E h$. This implies that we can sample from $\textrm{GB-RBM}(\lambda E, b, c)$ with a large $\lambda$ by the following procedure:
\begin{enumerate}
    \item Use a Gibbs sampler to sample $(x, h) \sim \textrm{GB-RBM}(\lambda' E, b, c)$, where $\lambda'$ is small.
    \item Set $y = x - \frac{1}{2} (\lambda' - \lambda) E h$.
\end{enumerate}
Therefore, assuming the Gibbs sampler is capable of generating faithful samples from $\textrm{GB-RBM}(\lambda' E, b, c)$ for some $\lambda' \geq 0$, this will produce faithful samples from $\textrm{GB-RBM}(\lambda E, b, c)$ for any $\lambda \geq 0$ large. In our experiments, we used $\lambda' = 0$.

\end{document}